%% file: IEEE-conference-template-062824.tex
\documentclass[conference]{IEEEtran}
\IEEEoverridecommandlockouts

\usepackage{cite}
\usepackage{amsmath,amssymb,amsthm,amsfonts, mathabx}
\usepackage{cleveref} 
\usepackage{graphicx}
\usepackage{textcomp}
\usepackage{xcolor}
\def\BibTeX{{\rm B\kern-.05em{\sc i\kern-.025em b}\kern-.08em
    T\kern-.1667em\lower.7ex\hbox{E}\kern-.125emX}}

\usepackage{natbib}
\usepackage{subfig}

\usepackage{url}
\newtheorem{theorem}{Theorem}
\newtheorem{corollary}{Corollary}[theorem]

\newtheorem{definition}{Definition}
\newtheorem{proposition}[theorem]{Proposition}

\usepackage{algorithm}
\usepackage{algpseudocode}

\usepackage{tikz}
\usetikzlibrary{trees}

 \newcommand{\af}[1]{\textcolor{black}{#1}}

\definecolor{tab_blue}{rgb}{0.122,0.467,0.706}
\definecolor{tab_orange}{rgb}{1.0,0.498,0.055}
\definecolor{tab_green}{rgb}{0.173,0.627,0.173}
\definecolor{tab_red}{rgb}{0.839,0.153,0.157}
\definecolor{tab_purple}{rgb}{0.580,0.404,0.741}
\renewcommand{\P}{\mathbb{P}}

\newcommand{\R}{\mathbb{R}}

\newcommand{\RNCM}{R_{\text{NCL}}}
\newcommand{\RC}{R_{\text{CC}}}

\newcommand{\Rsup}{R}
\usepackage{duckuments}
\usepackage{booktabs}
\usepackage{multirow}
\usepackage{rotating}
\usepackage{arydshln}
\usepackage{amssymb}
\usepackage{enumitem}
\usepackage{wrapfig}
\usepackage{pifont}

\begin{document}

\title{
Integrating uncertainty quantification into randomized smoothing based robustness guarantees
}

\author{\IEEEauthorblockN{Sina Däubener}
\IEEEauthorblockA{
\textit{Ruhr University Bochum}\\
Bochum, Germany \\
sina.daeubener@rub.de} \\
\IEEEauthorblockN{David Krueger}
\IEEEauthorblockA{
\textit{University of Cambridge}\\
Cambridge, England \\
david.scott.krueger@gmail.com}
\and
\IEEEauthorblockN{Kira Maag}
\IEEEauthorblockA{
\textit{Heinrich-Heine-University Düsseldorf}\\
Düsseldorf, Germany \\
kira.maag@hhu.de} \\
\IEEEauthorblockN{Asja Fischer}
\IEEEauthorblockA{
\textit{Ruhr University Bochum}\\
Bochum, Germany \\
asja.fischer@rub.de}
}

\maketitle

\begin{abstract}
Deep neural networks have proven to be extremely powerful, 
however, they are also vulnerable to adversarial attacks which can cause hazardous incorrect predictions in safety-critical applications.
Certified robustness via randomized smoothing gives a probabilistic guarantee that the smoothed classifier’s predictions will not change within an $\ell_2$-ball around a given input. 
On the other hand (uncertainty) score-based rejection is a technique often applied in practice to defend models against adversarial attacks. 
In this work, we fuse these two approaches by integrating a 
classifier that abstains from predicting when uncertainty is high into the certified robustness framework. 
This allows us to derive two novel robustness guarantees for uncertainty aware classifiers, namely (i) the radius of an $\ell_2$-ball around the input in which the same label is predicted and uncertainty remains low and (ii) the $\ell_2$-radius of a ball in which the predictions will either not change or be uncertain. 
While the former provides robustness guarantees with respect to attacks aiming at increased uncertainty, the latter informs about the amount of input perturbation necessary to lead the uncertainty aware model into a wrong prediction. Notably, this is on CIFAR10 up to 20.93\% larger than for models not allowing for uncertainty based rejection. We demonstrate, that the novel framework allows for a systematic robustness evaluation  of different network architectures and uncertainty measures and  to identify desired properties of uncertainty quantification techniques. Moreover, we show that leveraging uncertainty in a smoothed classifier helps out-of-distribution detection.
\end{abstract}

\section{Introduction}
\label{sec:intro}
In recent years, deep neural networks have demonstrated outstanding performance in a broad range of tasks such as image classification, speech recognition, and semantic segmentation~\citep{vit_model,Xu2023}. However, they have been consistently found to be vulnerable to small perturbations of the input, enabling the crafting of so-called adversarial examples~\citep{Szegedy_advattacks} which cause incorrect predictions.
These perturbations, which are designed to be not perceptible to humans, are hazardous in safety-critical applications, like automated driving or medical diagnosis, ultimately decreasing the trust in the model's predictions~\citep{Daza2021,carlini2023certified}. 

The potential threat that non-robust models pose to the deployment of machine learning systems in real-world scenarios has led to a huge amount of research on increasing adversarial robustness~\citep[e.g.][]{Klingner2020,hendrycks2019robustness} on the one hand and quantifying prediction uncertainty on the other~\citep[e.g.][]{Mackay1992,Gal2016,deep_ensemble}. 
While uncertainty quantification
increases the reliability of a model's predictions through confidence scores, works on improving adversarial robustness try to prevent
successful adversarial attacks. 
Naturally, there is a significant overlap 
between these research fields, where approaches try to 
detect adversarial examples based on uncertainty~\citep{Feinman2017,Maag2024} or leverage uncertainty to abstain from making a prediction~\citep{towards_robust_detection,Stutz_CCAT}. 
Unfortunately though, the robustness resulting from those approaches has to be evaluated empirically. The downside of such empirical robustness estimates is, that 
it is unclear if models remain robust against stronger attacks that evolve from the ever ongoing arms race between newly developed adversarial defenses and attacks.  
This question can be circumvented by certified robustness approaches~\citep[e.g.][]{cohen_adv_robustness, gowal_IBP, levinde_l1, NEURIPS2019_certified, SP10179303},
which provide mathematical guarantees that for a certain region around the input the predicted class label will not change.
So far however, such guarantees where not defined explicitly for uncertainty quantifying models.
To close the gap,  this work fuses uncertainty quantification with the state-of-the-art certified robustness via randomized smoothing (RS) approach~\citep{cohen_adv_robustness}. 
This allows to
derive two novel forms of robustness guarantees, which enable a cascade of further insights and synergies.
More precisely, we make the following contributions:
\begin{enumerate}
   \item We include uncertainty based rejection into the RS framework by leveraging an uncertainty-equipped classifier. This allows the transferal of future advances in uncertainty quantification into the certified robustness domain.
   \item We derive statistical guarantees in the form of two $\ell_2$-norm robustness radii. 
   The first specifies the radius of a ball around the input in which with high probability predictions will remain both consistent (i.e., the predicted class remains the same) and confident (i.e., have low uncertainty). The second gives the radius of a ball around the input in which with high probability predictions will either be consistent or indicate high uncertainty, and thus, an adversarial example can not mislead the model into a wrong and confident prediction.
   \item We modify the certification process, that is, the sampling process needed to estimate the robustness radii (those proposed in this paper as well as the original one proposed by \citet{cohen_adv_robustness}), such that we yield
   tighter approximations and thereby increased radii. 
   \item We theoretically and experimentally demonstrate the surprising fact, that incorporating uncertainty can even increase the robustness radii in which it is guaranteed that no class change occurs without any fine-tuning or retraining.
   \item We present an empirical evaluation of the novel radii with respect to different architectures and uncertainty measures. 
   \item We demonstrate that by incorporating uncertainty the smoothed classifier gets more robust to out-of-distribution data, i.e., data from outside the model's semantic space. 
\end{enumerate}

We present a brief review on related work in the next section, and introduce the concepts of certified robustness in \Cref{sec:background}. In \Cref{sec:including_uncertainty} we show how to include uncertainty into the certified robustness framework and in \Cref{sec:estimates} we derive a sample-based estimation procedure for the robustness radii. The numerical results are presented in \Cref{sec:experiments}, followed by limitations in \Cref{sec:limitations} and the conclusion in \Cref{sec:conclusion}.

\section{Related work}
\label{sec:related_work}
For a better understanding how our paper fills the gap between certified robustness via randomized smoothing and uncertainty quantification, a short overview over the related work in the different research domains is given below.

\subsection{Deterministic robustness guarantees}
There are different ways of deriving deterministic robustness guarantees, and one of them is the robustness through design by leveraging Lipschitz properties. If a model $f$ is trained in such a way that it is (locally) Lipschitz around an input $x_0$, it means that $|f(x_0)-f(x) | \leq L \cdot | x_o-x|$ holds for any $x \in \mathcal{B}(x_0)$, where $\mathcal{B}(x_0)$ is an area around $x_0$. This property is leveraged to bound the area in which it is guaranteed that for any $x$ the predicted class is the same as for $x_0$~\citep{hein_formal_2017, lipschitz_margin_training_2018}. There have been numerous works building on this idea by guaranteeing lipschitzness of different layer types~\citep{singla_skew_orth, orth_cayley_2021} and more complex architectures~\citep{ singla2022improved,  pmlr-v162-meunier22a, almost_orthogonal_2022}. 
Next to leveraging global Lipschitz bounds during training, \citet{Globally_robust_leino21a} introduced in their paper the idea that a model should abstain from making a prediction, if the sample lies within an $\epsilon$-tube around the decision boundary.

While this robustness by design is appealing, it has the downside that the training for smoothness  naturally decreases the flexibility of a neural network to class changes and thereby often decreases benign accuracy~\citep{lipschitz_margin_training_2018}.  
Another way of deriving deterministic robustness guaranties is through interval bound propagation techniques (IBP)~\citep[e.g.,][]{gowal_IBP, singh_ipb, mueller2023certified}, where the prediction margin is propagated through the network to identify the possibility of input derivations without changing the prediction outcome. However, also for these methods to lead to useful guarantees, the training objective needs to be changed to account for a trade-off between benign and verified accuracy.
\citet{sheikholeslami2021provably} extended the IBP literature, by proposing a way to jointly train a provably robust classifier and adversarial example detector, which includes the possibility of abstaining from making a prediction.
To the best of our knowledge, up to now the accuracies of models  with deterministic certified robustness guarantees are below those  of models with stochastic guarantees which will be explained in the next paragraph.

\subsection{Stochastic robustness guarantees}

The current state-of-the-art for certified robustness is based on stochastic guarantees off-springing from the works of~\cite{cohen_adv_robustness, lecuyer_certified_robust}, and \cite{NEURIPS2019_certified}. In those, one gets high probability guarantees, that the predicted 
class label does not change if the norm of the added perturbations does not exceed a certain threshold. Our work is based on~\cite{cohen_adv_robustness}, who leverage Gaussian randomized smoothing on top of any classifier (which turns the classifier into a random classifier by adding Gaussian noise to the input and averaging over predictions for such randomized input during inference), and hence their guarantee is agnostic to model training and architecture. 
There have been many follow-up works which propose ways to increase the certified robustness radius through e.g.~adversarial training~\citep{Salman}, mixing between benign and adversarial examples~\citep{jeong2021smoothmix}, data augmentation to account for low-frequency features~\citep{AdvRobustmeetsOOD-sun}, consistency training~\citep{consistency_jeong}, 
direct maximization of the certified robustness radius~\citep{Zhai2020MACER}, and using ensembles as base classifiers and distinguishing between certification and prediction~\citep{mark2022boosting}.
Simultaneously, the framework was extended, among others, to certify that the prediction is one of the top-k classes~\citep{jia2020certified_top_k}, to include invariances into the certification guarantee~\citep{schuchardt2022invarianceaware}, to transfer derived robustness guarantees of one model to an approximated model~\citep{ugare2024incremental}, and to allow for robust interpretability~\citep{levine}. 
Additionally, other works focused on deriving guarantees for further $\ell_p$-norm bounds \citep[e.g.,][]{higher_order_robustness, RS_of_all_shapes, tight_certificate_l0}.
Crucial to the RS approaches is that the probability mass for the predicted class of a smoothed classifier is substantially high. To ensure this, previous base classifiers were often trained with some kind of data augmentation. With the rise of diffusion models, this data augmentation training can be replaced by simply using a diffusion noising and denoising step before feeding the input into an ordinary classifier~\citep{carlini2023certified}. While appealing, it was shown that simply stacking the diffusion models and classifiers 
together often results in poor accuracy. To improve on this without extensive retraining of the whole architecture,  \citet{sheikholeslami2022denoised} proposed to add a low-cost classifier to the stacked solution, which reduces the amount of misclassification by deciding to abstain from making a prediction and thereby increasing the certified robustness of these stacked models. 
While the formula of their rejection-base-classifier is very similar to our uncertainty equipped classifier in spirit,  their rejection is based on an additionally trained rejection classifier.  In contrast, our definition of an uncertainty equipped classifier is based on a more general uncertainty function and
includes their approach that results from a specific choice of uncertainty function (namely that given by the rejection classifier). More importantly, their rejection based classifier is used to define a smoothed classifier that does not allow for rejecting the classification. Instead the goal is to derive an alternative way to estimate the robustness radius which can lead to an increase. 
A more in-depth discussion on how the estimation strategy differs is given in \Cref{sec:estimates}.


\subsection{Uncertainty quantification} 
Methods for uncertainty quantification aim at assessing and depicting a model's confidence in its predicted class $y$ often either directly through the likelihood 
or related quantities such as the prediction entropy, 
in case of classification problems. 
It has been shown that standard neural networks architectures tend to be overconfident, which means that they have low entropy and high likelihood even on misclassified or out-of-distribution (OOD) samples~\citep{calibration_guo, hein_relu_overconfident}. 
To fix this, a multitude of methods has been proposed. Deterministic methods for example often include the usage of OOD samples during training~\citep{NEURIPS2021_118bd558} or leveraging ensembles such that predictions (ideally) disagree on OOD samples \citep{pmlr-v180-kumar22a}. 
On the other hand, many probabilistic approaches were proposed, for example by modeling an output distribution over the labels~\citep{Gal2016}, applying  Gaussian processes or Kernel methods on top of neural network architectures~\citep{NEURIPS2022_eb7389b0}, or leveraging Bayesian neural networks which capture model uncertainty due to finite data in their posterior distribution~\citep{Mackay1992}. For a detailed review please be referred to~\citep{ABDAR2021243}. Interestingly though, there are multiple parallels to adversarial robustness research: training with OOD data is also some form of data augmentation and adding Lipschitz constrains to the training seems also to improve the model's confidence and capability to distinguish between in and out of distribution data \citep{pmlr-v216-ye23a}.

\subsection{Score-based detection/rejection}
There have been numerous works that try to detect OOD and/or adversarial examples based on network features~\citep{Feature_squeeze, boosting_ood_feature}, or uncertainty measures like prediction variance~\citep{Feinman2017}, entropy~\citep{Smith2018, robust_detection_2018}, or simply the softmax probabilities~\citep{Hendrycks2016}. 
However, such detection methods can be used for either crafting ``uncertainty attacks''~\citep{uncertainty_attacks} or circumvented in general if the attacker has knowledge of the rejection method and accounts for it during their attack calculation~\citep{obfuscated-gradients}. 
For a score-based rejection model to be robust to the latter, consistent low confidence around benign inputs is needed such that high-confidence adversarial examples are impossible. This observation is the key element for confidence-calibrated adversarial training~\citep{Stutz_CCAT}, where the network is trained to decrease its output probabilities in the direction of an adversarial example successively to the uniform distribution. This training approach leads to more robust models even with regard to attacks based on different $l_p$-norms.

\section{Background}
\label{sec:background}
In this section we give a brief formal overview about the certified robustness via randomized smoothing approach proposed by \citet{cohen_adv_robustness} and the sampling procedure used for estimating the robustness bound.

\subsection{A robustness guarantee for smoothed classifiers} Certified robustness of a classification model is a mathematical guarantee that for all perturbed inputs
$x+ \delta$, with $\delta\in \mathcal{B}$, it holds that the model predicts the same class as for $x$.
\citet{cohen_adv_robustness} derived the state-of-the-art definition for certified robustness of smoothed classifiers in their main theorem which is restated below.

\begin{theorem}[\citet{cohen_adv_robustness}]
\label{thm:cohen}
Let $\mathcal{Y}$ be a discrete and finite label space and $f:\R^d \rightarrow \mathcal{Y}$ be any deterministic or stochastic function. Let $\epsilon \sim \mathcal{N}(0, \sigma^2 I) $ and  $g(x)  = \arg \max_{c \in \mathcal{Y}} \P(f(x+ \epsilon) = c)$. Suppose $c_A \in \mathcal{Y}$ and $p_A, p_B \in [0,1]$ satisfy
\[
\P(f(x+ \epsilon) = c_A) = p_A \geq p_B =\max_{c \in \mathcal{Y}:c \neq c_A} \P(f(x+ \epsilon)=c) \enspace .
\]
Then $g(x+ \delta) = c_A \  \forall \ \delta \text{ with } \| \delta\|_2 < R$, where $R = \frac{\sigma}{2} ( \Phi^{-1}(p_A) - \Phi^{-1}(p_B))$ and $\Phi^{-1}(\cdot)$ is the inverse standard Gaussian CDF\footnote{The original theorem is stated with using lower and upper bounds on $p_A \geq \underline{p_A} \geq \overline{p_B} \geq p_B$ since the exact probabilities can not be directly calculated (but arbitrary well approximated with an increasing number of Monte Carlo samples). For better readability, we introduce these bounds only at a later point when we actually calculate them.}.
\end{theorem}
While the theorem holds for very general functions, $f$ is usually considered to be a deterministic classification network, which
we will for brevity refer to as
\emph{base classifier}.
The stochastic classifier $g$ resulting from adding Gaussian noise to the input of the base classifier is called \emph{smoothed classifier}.
Moreover, throughout the paper we will 
use $p_A$ for the probability mass assigned by the smoothed classifier to the most likely class,
i.e., $p_A =  \max_{c \in \mathcal{Y}} \P(f(x+ \epsilon) = c) =  \P(f(x+ \epsilon) = c_A)$ and denote with $p_B =\max_{c \in \mathcal{Y}: c \neq c_A} \P(f(x+ \epsilon)=c) $ the probability mass of the ``runner-up class'' which is the second most likely one.
The above theorem certifies that the prediction of the smoothed classifier $g(\cdot)$ for an input $x$ is invariant or consistent under an input shift of $\delta \in \mathcal{B}$ with $\mathcal{B} = \{ \delta \; | \; \| \delta \|_2 \leq R = \frac{\sigma}{2} ( \Phi^{-1}(p_A) - \Phi^{-1}(p_B))\}$. This consistency requirement can formally be phrased as:

\begin{definition}[Consistent predictions]
    We say a classifier $g:\mathbb{R}\rightarrow \mathcal{Y}$ is consistent in its prediction $\af{ g(x)} =c_A$  in an area $\mathcal{B}$, 
    if
\begin{equation*}
        \forall \ \delta \in \mathcal{B}: \ \arg\max_{c\in\mathcal{Y}} g(x+ \delta) =c_A \enspace .
    \end{equation*}
\end{definition}

\citet{cohen_adv_robustness} further show in Theorem~2 of their paper, that if $p_A + p_B \leq 1$ the derived bound is tight, in the sense that, if we have a perturbation which exceeds $R$, then there exists a classifier $h(\cdot)$ which leads to the the same probability mass under Gaussian noise as $f(\cdot)$, but will predict a class other than $c_A$ for $x+\delta$ with $\| \delta \|_2 >R$. A result similar to Theorem \ref{thm:cohen} but with a different proof strategy was derived by~\citet{Salman}, where it is phrased that $R$ is the minimal bound on the required $\ell_2$ perturbation which leads to a label flip from class $c_A$ to (any) class $c_B$.

\subsection{Sample based estimation of the robustness guarantee}
For actually turning  Theorem \ref{thm:cohen} into a guarantee in practice, one needs to determine $p_A$ and $p_B$, which is in general not feasible in closed form.  Therefore, \citet{cohen_adv_robustness} proposed a statistical way to estimate an lower bound $\underline{p_A}$ for $p_A$ and an upper bound $\overline{p_B}$ for $p_B$, which hold with a probability of $1- \alpha$. Their certification process is based on a two step sampling approach: First, the class labels of $n_0$ noisy input variations are predicted and the amounts of occurrences of each class are counted. Then, the test of~\cite{Hung2016RankVF} is used to abstain from making a prediction in cases where one can not guarantee with probability $1- \alpha$, that $c_A$ is the class with the highest probability mass. 
Based on the prediction counts of another set of $n_1$  draws, 
they set $\underline{p_A}$ to the lower confidence bound of a binomial distribution with $n_A$ observations for class $c_A$, $n_1$ total draws, and confidence level $1- \alpha$. Lastly, instead of calculating $\overline{p_B}$, they approximate this probability by $1-\underline{p_A}$.


%
\section{Including uncertainty into certified robustness}
\label{sec:including_uncertainty}

\begin{figure*}[htb]
    \centering
    \subfloat[With $\theta = \text{sup}$.]{\includegraphics[width= 0.27\textwidth]{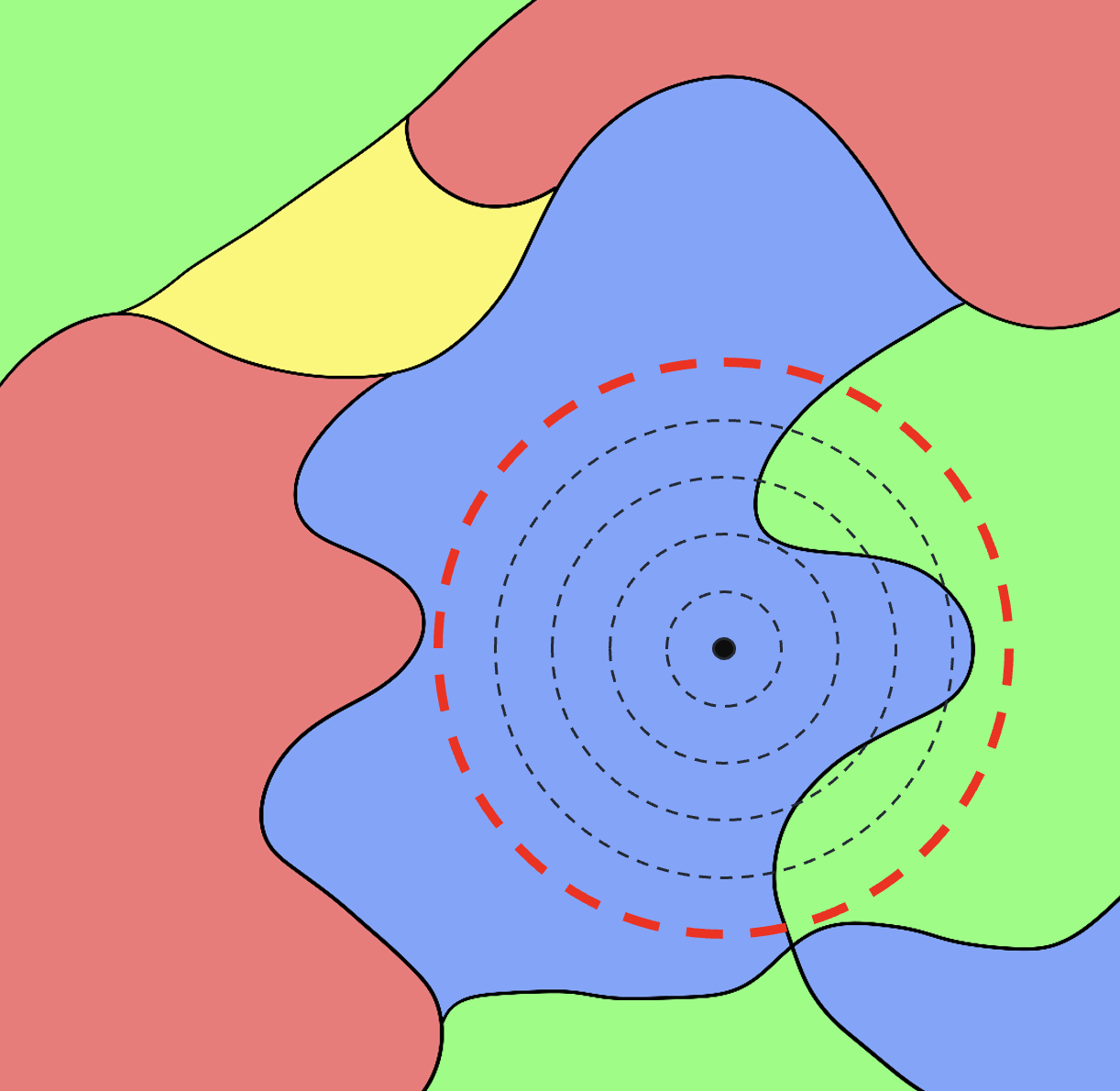}}
\qquad
\subfloat[Symmetric uncertainty.]{\includegraphics[width= 0.27\textwidth]{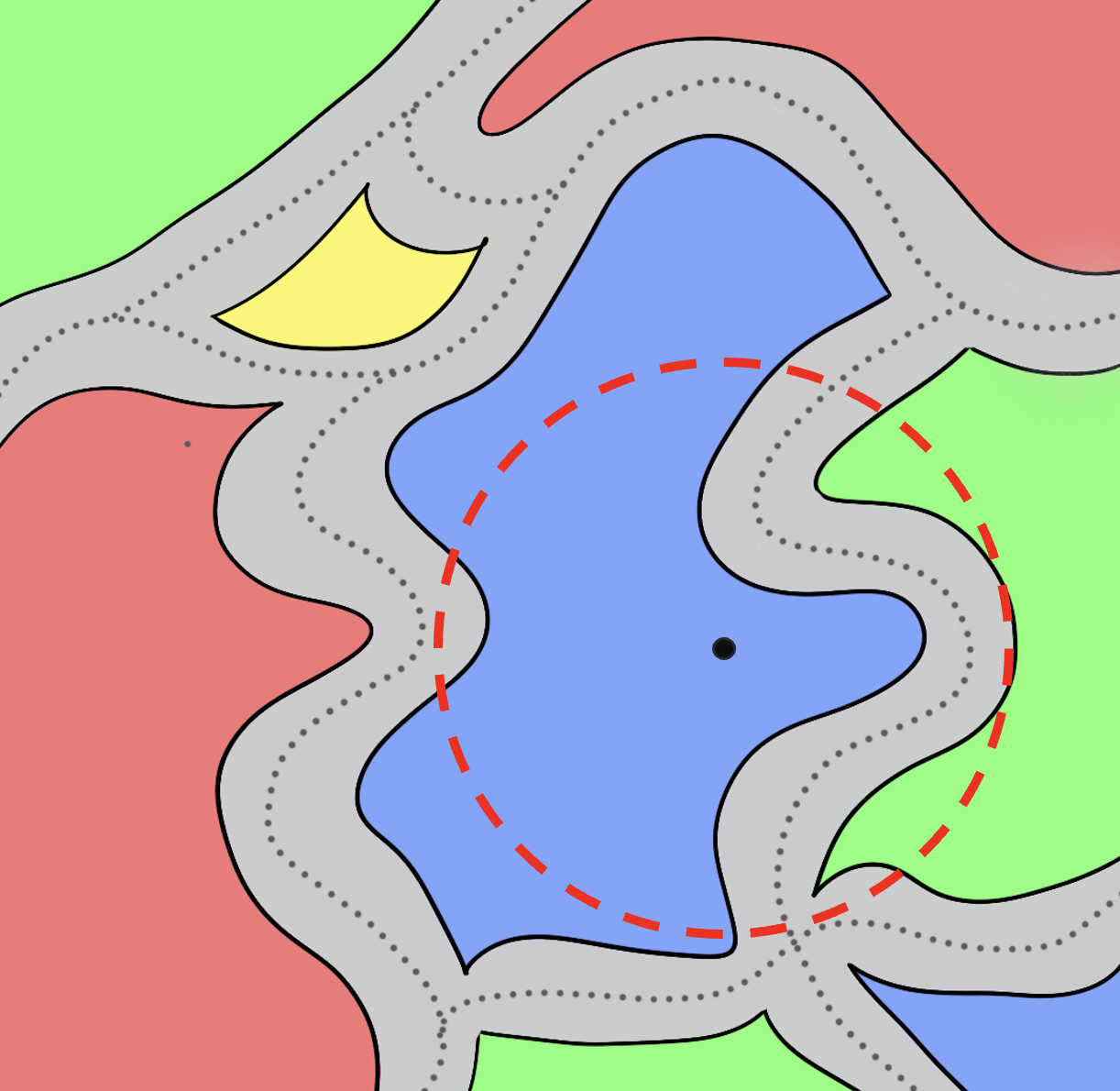}}
\qquad
\subfloat[Asymmetric uncertainty.]{\includegraphics[width= 0.27\textwidth]{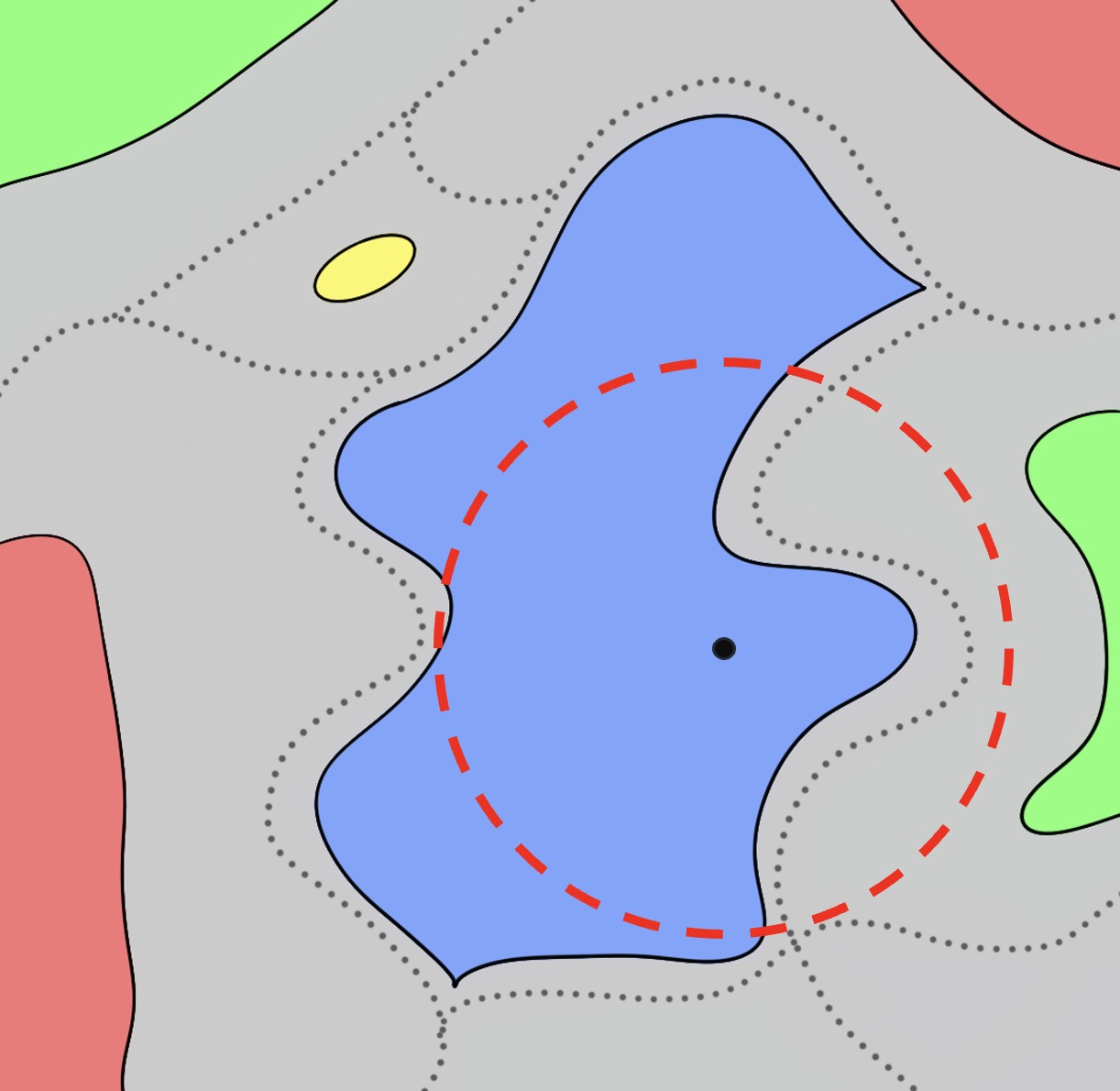}}
    \caption{\textbf{Decision regions of a classifier $f$ } depicted in different colors, where the black dot marks the position of an input $x$ and the circles the different probability levels of a Gaussian distribution centered at $x$: (a) is the image recreated based on~\citet{cohen_adv_robustness} where the class with the most probability mass under Gaussian noise is predicted (here the blue class).  
    (b) and (c) show different uncertainty distributions around the original decision boundaries (dotted) which lead to different uncertainty regions depicted in gray. While
    the uncertainty behavior around decision boundaries in (b) is almost symmetric, 
    image (c) shows a strong asymmetric uncertainty behavior, making a confident misclassification difficult. 
    }
    \label{fig:example_bounds}
\end{figure*}

We now extend this framework to derive robustness guarantees for uncertainty quantifying classifiers that allow to abstain from making an prediction if the uncertainty exceeds a predefined threshold. Intuitively, the robustness of such an classifier can be characterized on the one hand by the area around inputs in which the prediction does not change and certainty remains high, and on the other hand by the area around an input in which it is guaranteed that the prediction does not change, that is, the classifier either predicts the same class or shows high uncertainty.

To derive such guarantees, let us start by formally defining an uncertainty quantifying function.

\begin{definition}[Uncertainty function]
\label{def:uncertainty}
Given a classifier $f: \mathcal{X} \rightarrow \mathcal{Y}$, where $\mathcal{Y}$ is a set of classes, a function $u_f: \mathcal{X} \rightarrow \mathcal{U}$ that maps an input $x \in \mathcal{X}$ to a score $u_f(x)$ is called an uncertainty function of $f$
if a high value of
$u_f(x)$ is
indicative for a false classification of $f$\footnote{The negation of this definition would equal the definition of a confidence score~\citep{geifman2018biasreduced}.}.
\end{definition}

Commonly used functions for quantifying uncertainty in neural networks fall under this definition.
To give examples for such functions, let us first recall, that neural network classifiers usually output logits that are turned in a categorical probability distribution $p(\cdot|x)$ over the classes by the softmax function. 
The neural network prediction is than given by $f(x)=argmax_{y \in \mathcal{Y} p(y|x)}$ and the most simple uncertainty function that can be taken into account is the \textbf{prediction confidence} which is just given by the probability of the most likely (and thus predicted) class, i.e., by
$$max_{y \in \mathcal{Y}} p(y|x)\enspace.$$
Another frequently uses uncertainty function for neural networks is
the \textbf{negative prediction margin} between the two most likely classes $y_A$ and $y_B$, that is, $$-( p(y_A|x)- p(y_B|x) \enspace,$$ which is related to the closes distance to a decision boundary. Last but not least,
the \textbf{entropy} defined by 
\begin{equation}\label{eq:entropy}
    H(x) = - \sum_{c \in \mathcal{Y}} \log p(y_c|x) \enspace ,
\end{equation} 
is often employed as uncertainty function. 
These quantities can easily be computed for any neural network with softmax output neurons. 
However, Definition \ref{def:uncertainty} also allows for more advanced methods such as for example approximate Bayesian methods like Monte Carlo dropout \citep{MonteCarloDropout} or Stochastic Weight Averaging Gaussian \citep{SWAG}.


Given the uncertainty function of a model, we can formally define a confident prediction as following:
\begin{definition}[Confident prediction]
Given a classifier  $f: \mathcal{X} \rightarrow \mathcal{Y}$ with an uncertainty function $u_f(s)$ and a predefined threshold $\theta$.  A prediction of  $f$ is said to be confident under $u_f(s)$ with a threshold $\theta$, if $u_f(s) \leq \theta$.
\end{definition}
The threshold $\theta$ is usually chosen in such a way that the  
accuracy is only slightly above the accuracy given when only confident predictions are taken into account~\citep{Stutz_CCAT}. 
Merging the above definitions, an uncertainty 
equipped classifier is given as:
\begin{definition}[Uncertainty-equipped classifier]
Let  $f: \mathcal{X} \rightarrow \mathcal{Y}$ be a classifier 
and  $u_f: \mathcal{X} \rightarrow \mathcal{U}$ an uncertainty function on $f$. Then the corresponding uncertainty-equipped classifier $f^*_{u_f, \theta}: \mathcal{X} \rightarrow \mathcal{\tilde Y}$ with threshold $\theta$ is defined by 
\begin{equation}
       f^*_{u_f, \theta}(x) =  \begin{cases}
        c_{\theta}, & \text{if } f(x) = c \in \mathcal{Y} \enspace \text{and} \enspace u_{f}(x)< \theta \\
        v_{\theta}, & \text{if }  u_{f}(x) \geq \theta \enspace ,
        \end{cases} 
\end{equation}
where $ c_{\theta} \in \mathcal{Y}_{\theta}$ for any $ c \in \mathcal{Y}$ is referred to ``class $c$ with confidence'', $\mathcal{\tilde Y} = \mathcal{Y}_{\theta} \cup \{ v_{\theta}\}$,  and 
 $v_{\theta}$ is called the \textbf{uncertainty class}. 
\end{definition}

Note, that if we set $\theta$  to any value larger or equal to
$\sup (\mathcal{U})$, it holds that $\{x \in \mathcal{X}| f^*_{u_f,\theta}(x) = c_{\theta} \} = \{x \in \mathcal{X}| f(x) = c \} $ and the class $v_{\theta}$ will be empty. 
That is, we recover the original function $f$.
This situation is depicted in Figure~\ref{fig:example_bounds} (a), where only colored regions are present, representing different decision regions of a classifier $f$. For (b) and (c) 
$\theta$ is sufficiently smaller than $\sup (\mathcal{U})$, such that the gray-colored uncertainty regions appear as decision areas where $u_{f}(x)$ is above the threshold $\theta$ and the uncertainty class is predicted.

By applying the certified robustness via randomized smoothing framework to the uncertainty-equipped classifier we now  obtain the novel robustness guarantees that allow us to evaluate the robustness gained by uncertainty-based rejection.

\begin{corollary}[Certified robustness with uncertainty]
\label{cor:uncertainty_robust}
Let $f^*_{ u_f, \theta}:\R^n \rightarrow \mathcal{\tilde{Y}}$ be an uncertainty equipped classifier as defined above. 
Let $\epsilon \sim \mathcal{N}(0, \sigma^2 I)$ and $g$ be a
smoothed classifier given by
\begin{equation}
    g(x) = 
    \arg \max_c \P ( f^*_{u_f, \theta}(x+\epsilon) = c ) 
    \enspace.
\end{equation}
Suppose that for a given $x \in \mathcal{X}$ and some $c_{A, \theta} \in \mathcal{\tilde{Y}}$ it holds $g(x)=c_{A, \theta} \neq v_{\theta}$ and that the runner up class is defined by $c_{B, \theta} = \arg \max_{c\neq \{ c_{A, \theta}, v_{\theta} \}} \P ( f^*_{ u_f, \theta}(x+\epsilon) = c ) $. For simplifying notation we set $p_{c_{\theta}} (x) = \P ( f^*_{u_f, \theta}(x+\epsilon) = c_{\theta} )$. Then it holds that:

\begin{enumerate}[label=(\alph*)]
     \item The radius in which we have a \textbf{consistent and confident (CC)} prediction for the current predicted class is given by $\RC = \frac{\sigma}{2} \cdot \left( \Phi^{-1}\left(p_{c_{A, \theta}}(x)  \right) - \Phi^{-1}\left(\max_{c \neq c_{A, \theta}}p_c(x) \right)\right)$ such that $c$ is the second most likely class other than $c_{A, \theta}$, which explicitly includes the uncertain class. That means, that  $\forall \delta$ with $\| \delta\|_2 < \RC$ it holds $ g(x+ \delta) = c_{A, \theta}$.
    \item The radius within which there are \textbf{no confident label change (NCL)} by the smoothed classifier 
    is given by:  $\RNCM = \frac{\sigma}{2} \left( \Phi^{-1}\left(p_{c_{A, \theta}}(x) + p_{v_\theta}(x) \right) - \Phi^{-1}\left(p_{c_{B, \theta}(x)}\right)\right)$, that is $g(x+ \delta) \in \{ c_{A, \theta} 
    ,v_{\theta} \} \ \forall \| \delta\|_2 < \RNCM $, with $\Phi^{-1}(\cdot)$ as in Theorem~\ref{thm:cohen}.
\end{enumerate}
\end{corollary}
\begin{proof}
    (a) is equal to the original Theorem 1 of \citet{cohen_adv_robustness} when using an uncertainty equipped classifier as $f$.\\
    For proving (b) we only need to make slight modifications: In the proof of Theorem 1 by \citet{cohen_adv_robustness} they set $\underline{p_A} \leq \P(f(x+ \epsilon) = c_A)$. For proving our claim, we just need to redefine $\Tilde{p}_A = p_{c_{A, \theta}}(x) + p_{v_{\theta}}(x) = \P(f(x+ \epsilon) = c_{A, \theta} \text{ or } f(x+ \epsilon) = v_{\theta} ) $. 
    Per definition we know, that $p_{c_{A, \theta}}(x) \geq p_{c_{B, \theta}}(x) $, therefore $p_{c_{A, \theta}}(x) + p_{v_{\theta}}(x) \geq p_{c_{B, \theta}}(x) $ which is a necessity for the theorem. The rest of the proof is identical to \citet{cohen_adv_robustness}, where one needs to replace $\underline{p_A}$ by $\Tilde{p}_A $ or its lower bound everywhere. 
\end{proof}

While a high $\RC$ is a guarantee against all attacks including possible uncertainty attacks~\citep{uncertainty_attacks} a high $\RNCM$ is desirable where confident label changes can have disastrous effects as for example mistaking stop signs in autonomous driving. 
The trivial case, where every input to the base classifier $f$ is classified as uncertain, leads to the highest $\RNCM$, but also to the lowest $\RC$ and accuracy. This is undesirable, and ideally, both $\RC$ and $\RNCM$ should be high.
Another consequence of the corollary regards the ordering of the certified robustness radii, that is,
\begin{equation}\label{eq:ordering}
    \RC \leq \RNCM \enspace ,
\end{equation}
and equality of $R_{\text{CC}, \text{sup}} = R_{\text{NCL}, \text{sup}} $ which we will simply denote as $R$ to avoid confusion. 
It is not directly clear, however, how $\Rsup$ and $\RC$ relate to each other. Intuitively, one may think, that including uncertainty should decrease the robustness radius. However,
including uncertainty does impact both quantities needed in the calculation of the radius. This allows for situations where the robustness radius stays indeed the same (but guarantees confidence as well) or even increases, which is specified in the following corollary.

\begin{corollary}
The certified robustness radius guaranteeing
a consistent and confident prediction for a smoothed uncertainty-equipped classifier is greater than or equal to the robustness radius guaranteeing a consistent prediction in the original setting (without leveraging the uncertainty of the base classifier)
when using the same amount of noise $\sigma$
i.e., $\RC > R$, iff 
\begin{align}
     &\Phi^{-1} (\max_{c_{\text{sup}} \neq c_{A, \text{sup}}}p_{c_{\text{sup}}})  - \Phi^{-1} (\max_{c\neq c_{A, \theta}} p_c)  \label{eq:p_b_difference}\\
      & > \Phi^{-1} (p_{c_{A, \text{sup}}}) - \Phi^{-1} (p_{c_{A,\theta}})  \enspace. \label{eq:better_with_uncertainty}
\end{align}
\end{corollary}
\begin{figure}
  \begin{center}
\includegraphics[trim=0 10 0 10,clip,width= 0.35\textwidth]{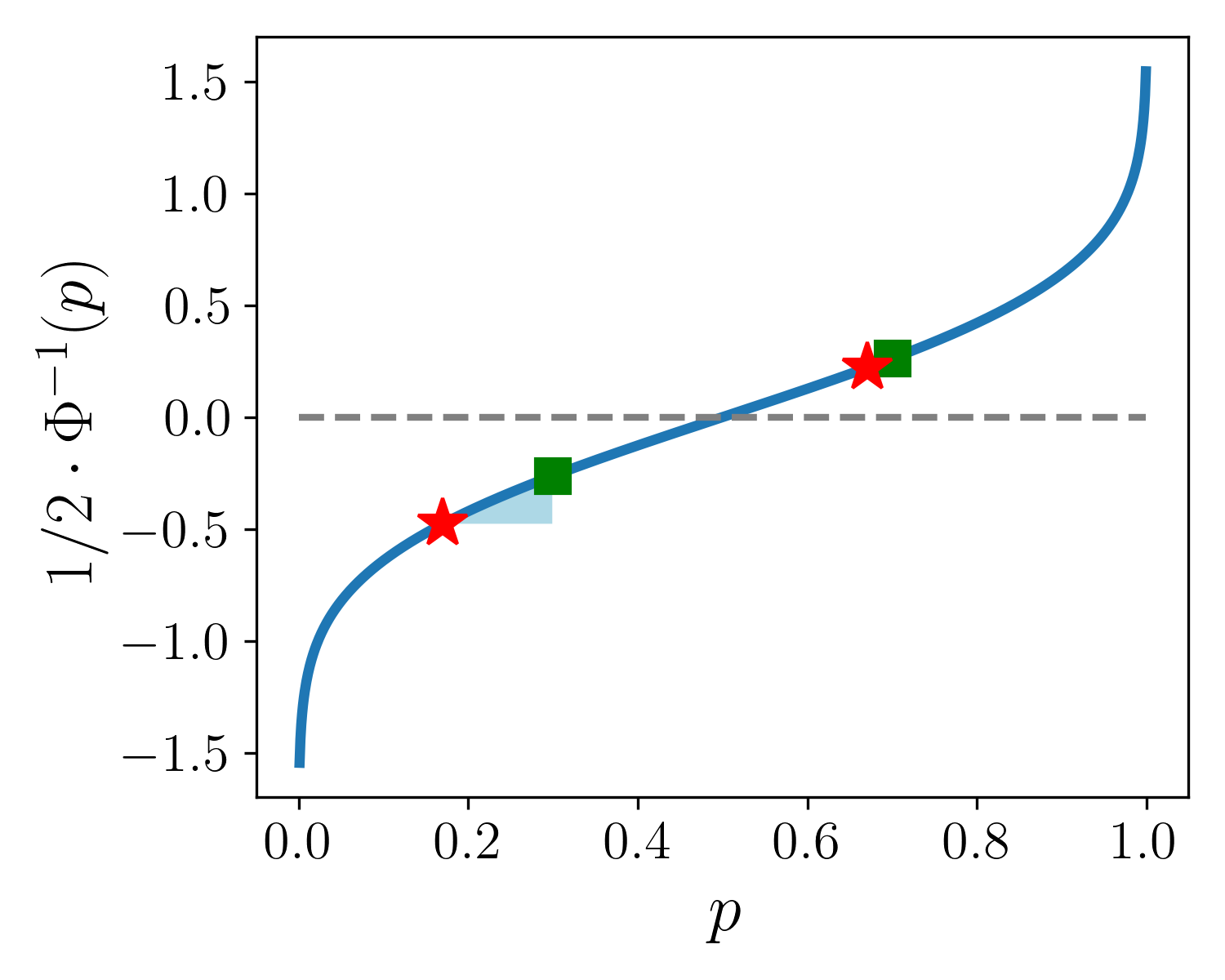}
  \end{center}
  \caption{
\textbf{  An illustration of why $R_{\text{CC}, \theta}$ can provide a stronger guarantee under appropriate conditions.}
The plot shows $\frac{1}{2}\cdot \Phi^{-1}(p)$ in dependence of $p$. (\LARGE{\textcolor{tab_green}{$\sqbullet$}}\normalsize) denote values belonging to $p_{c_{A, sup}} $and $p_{c_{B, sup}}$, whereas (\textcolor{tab_red}{$\bigstar$}) represents values w.r.t. $p_{c_{A, \theta}} $ and $\max_{c \neq c_{A, \theta}} p_{c} $. The bigger decrease of the runner-up class probability with uncertainty results in a higher $y$-axis difference which is here equal to the certified robustness radii.
  }
\label{fig:toy_and_artifacts}
\end{figure}
This corollary, which directly follows from rearranging the terms in inequality $\RC > \Rsup $, relates the change of class probability masses due to introduced uncertainty to improved certified radii. It shows for example, that if $p_{c_{A, \text{sup}}}$ is high, than 
considering $p_{c_{A, \theta}}$  instead of  $p_{c_{A, \text{sup}}}$ can hardly lead to a robustness increase as the exponential behavior of $\Phi^{-1}(\cdot)$ in the tails (compare Figure~\ref{fig:toy_and_artifacts}) increases \cref{eq:better_with_uncertainty}. 
However, the reversed statement holds true for the runner-up probability, such that 
\cref{eq:p_b_difference} is increased making a robustness advantage of with leveraging uncertainty more likely. 
To sum up, these observations can be leveraged to improve the certified robustness guarantee with uncertainty $RC$ over $R$ if either $p_{c_{A, \text{sup}}}$ is low and/or uncertainty regions are asymmetric  around  decision boundaries at  the expense of the second most likely class.
This latter observation is particularly interesting as it describes the ideal uncertainty behavior around decision boundaries for increased robustness in neural networks, which is still unexplored. 

\section{Sample based estimation of the robustness radii}
\label{sec:estimates}
In this section we describe how we empirically derive the class probabilities and the resulting robustness guarantees. 
Since we can not estimate these quantities analytically, we aim at a sample-based robustness certificates that hold with an $\alpha$ error of $0.001$ (as used in previous works). 

\subsection{Description of the sampling procedure.}
Our approximation procedure
is similar to that used by~\citet{cohen_adv_robustness} with pseudo-code being provided in Algorithm~\ref{alg:predict} and \ref{alg:certify}.
To allow for a fair comparison of results we first describe how our approach differs for the original estimation of the certified robustness radius and detail later how its applied for estimating $\RNCM$ and $\RC$. 

One major difference compared to our work is that 
\citet{cohen_adv_robustness} 
make use of the inequality 
$$
p_{c_{B, \theta}}= 1- \sum_{c \in \mathcal{Y}: c \neq c_B} \P(f(x+ \epsilon)=c) \leq 1- p_{c_{A, \theta}} \leq 1- \underline{p_{c_{A, \theta}}}  
$$ 
and approximate $p_{c_{B, \theta}}$ by $1- \underline{p_{c_{A, \theta}}}$. The reason for this simplification was that they reported to have mostly observed one class or only one other class
when approximating the class probabilities based on  $n_0=100$ samples. 
However, by increasing the initial sample size to $n_0=1,\!000$ we observe that noisy input variations of the same image are more frequently assigned different class labels, as can be seen in Figure~\ref{fig:num_classes} for CIFAR10. 
\begin{figure}[t]
\centering
\includegraphics[width=0.95\linewidth]{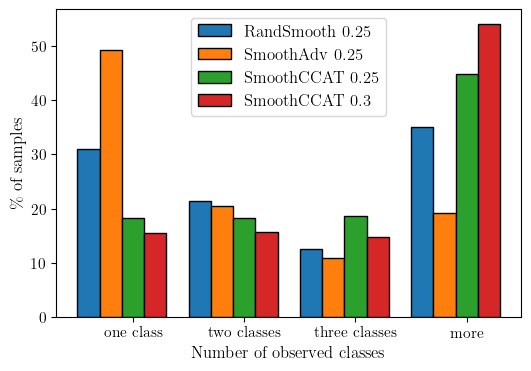}
    \caption{\textbf{Number of different assigned class labels to noisy versions of the same input for CIFAR10.} The one-vs-all method of~\citet{cohen_adv_robustness} is suboptimal when examples have multiple neighboring classes;
we observe this is typically the case when using $n_0=1,\!000$.}
    \label{fig:num_classes}
\end{figure}
Results on ImageNet are similar and depicted in \Cref{sec:appendix_experiments}.
Therefore, we altered the estimation procedure of Cohen et al. to 
leverage the presence of multiple class labels for better approximations. 
To do so, we first sample $n_0$ noisy images, predict their class label and collect the class counts in \texttt{counts}. 
We then conduct the unadjusted pairwise test introduced by~\citet{Hung2016RankVF} for identifying the winning class and the runner-up class with confidence $1-\alpha$. If we do not get a winning class, we abstain from making a prediction. If we do not get a runner-up class with the specified confidence value, we simply use the one-vs-all approach as in~\citet{cohen_adv_robustness}. The high-level decision structure is depicted in~\Cref{alg:first_step}\footnote{The increase of the initial sample size from $n_0 = 100$ to $n_0=1,\!000$ was made for deriving statistical significant runner-up class estimates in around 30\% of the cases.}.

\begin{figure}[tb]
\centering
    \begin{minipage}{.805\linewidth}
    \rule{\linewidth}{0.4pt} \\
    \textbf{First step in \textsc{Certify}: Get class estimates}   \\ 
    \vspace{-1ex} 
    \hrule
\begin{tikzpicture}
\node [align=left] [xshift = -0.3cm]  { 
    \texttt{counts} $\leftarrow$ \textsc{SampleUnderNoise}($f,x,  \sigma, n_0$)\\
    $n_A, n_B, n_C \leftarrow $ \# top three entries in \texttt{counts} \vspace{2pt} \\ 
    \phantom{12345}\textsc{BinomPValue}($n_A, n_A + n_B)\leq \alpha$} [sibling distance = 3.5cm, level distance = 1.3cm]
    child [red] {node[xshift = -0.9cm] {\textsc{Abstain}}} 
    child {node[align=center][yshift = -.3cm] [xshift = -0.9cm]{\textbf{return} $\hat{c}_A$ \\ \textsc{BinomPValue}($n_B, n_B + n_C)\leq \alpha$}
    child [red] {node[xshift = -0.9cm] {\textsc{one-vs-all}}}     
    child {node[xshift = -0.9cm] {\textbf{return} $\hat{c}_B$} 
    edge from parent [blue]
        node[right] {yes}}
    edge from parent [blue] node [right] {yes} };
\end{tikzpicture}  
    \end{minipage}
\caption{
Our method only resorts to one-vs-all when we cannot confidently identify a runner-up class. Functions used are identical to the ones from~\citet{cohen_adv_robustness}.} 
    \label{alg:first_step}
\end{figure}

Given the $\hat{c}_A$ and $\hat{c}_B$ 
estimated as the most probable and the runner-up class by our sampling procedure,
the upper and lower bounds of $p_{A}$ and $p_{B}$ are derived by calculating (Bonferroni corrected) Binomial confidence intervals with the Clopper-Pearson test~\citep{clopper_pearson} for each probability separately on another set of $n=100,\!000$ noisy samples\footnote{For computational reasons we use only $n=10,\!000$ for SmoothViT as done by~\citet{carlini2023certified}.}. For cases where no runner-up class was given, we resume to the approach of \citet{cohen_adv_robustness}.
The full estimation procedure is described by \Cref{alg:predict} and \Cref{alg:certify}, where the former estimates the class ranking which is used in the latter for estimating the certified radius. 
\textsc{SampleUnderNoise}$(f, x, \sigma, n_0)$ is a function creating $n_0$ noisy images $x+ \epsilon$ with $\epsilon\sim \mathcal{N}(0, \sigma ^2)$, using classifier $f$ to predict the class label of those, 
and returning a vector \texttt{counts}, which stores the number of class counts for each possible class;
\textsc{BinomPValue}$(n_1, n)$ calculates the p-value for a Binomial distribution used for testing if the probability of success is bigger than $0.5$ 
where $n_1$ is the number of ``successful'' trials
and $n$ the number of total observations;  and
\textsc{Upper/LowerConfInterval}$(n_A, n, \alpha)$ returns the upper/lower confidence bound for the probability of success of a Binomial distribution when observing $n_A$ successes with the Clopper-Pearson confidence interval.

\begin{algorithm}[t]
\caption{\textsc{Predict}$(f, n_0, \alpha, \sigma)$}
\begin{algorithmic}
\State \texttt{counts}$ \gets $ \textsc{SampleUnderNoise}$(f, x, \sigma, n_0)$
\State $n_A, n_B, n_C \gets$ top three entries in \texttt{counts} 
\If{\textsc{BinomPValue}$(n_A, n_A+n_B) \leq \alpha$}
    \Return predict $\hat{c}_A$, which is the class belonging to $n_A$
    \If{\textsc{BinomPValue}$(n_B, n_B+n_C) \leq \alpha$}
        \Return runner up class $\hat{c}_B$, which belongs to $n_B$
    \Else{ $\hat{c}_B$= NaN}
    \EndIf
\Else{ $\hat{c}_A$= NaN}
\EndIf
\end{algorithmic}
\label{alg:predict}
\end{algorithm} 

\begin{algorithm}
\caption{\textsc{Certify}$(n_0, n, \alpha, \sigma, f$)}
\begin{algorithmic}
\State $\hat{c}_A$, $\hat{c}_B$ $\gets$ \textsc{Predict}$(f, n_0, \alpha, \sigma)$
\If{$\hat{c}_A \neq $ NaN}
    \State \texttt{counts}$ \gets $ \textsc{SampleUnderNoise}$(f, x, \sigma, n)$
    \State $n_A \gets $ \texttt{counts}$[\hat{c}_A]$
    \If{$\hat{c}_B \neq $ NaN}
        \State $n_B \gets $ \texttt{counts}$[\hat{c}_B]$
        \State $\underline{p_A} \gets $ $\textsc{LowerConfInterval}( n_A, n, \alpha/2)$ 
        \State $\overline{p_B} \gets $ $\textsc{UpperConfInterval}( n_B, n, \alpha/2)$ 
        \If{$\underline{p_A} + \overline{p_B} > 1$}
            \State $\underline{p_A} \gets $ $\textsc{LowerConfInterval}( n_A, n, \alpha)$ 
            \State $\overline{p_B} \gets  1 - \underline{p_A}$ 
        \EndIf
    \ElsIf{$\hat{c}_B = $ NaN}
    \State $\underline{p_A} \gets $ $\textsc{LowerConfInterval}( n_A, n,\alpha)$ 
    \State $\overline{p_B} \gets  1 - \underline{p_A}$ 
    \EndIf
\Else{ Abstain}
\EndIf \\
\Return $R=\frac{\sigma}{2} (\Phi^{-1}(\underline{p_A}) - \Phi^{-1}(\overline{p_B}))$ \textbf{if} $R>0$
\end{algorithmic}
\label{alg:certify}
\end{algorithm}

In \Cref{subsec:eval_strategy} we evaluate how this changed approximation procedure impacts the certified robustness radius of Cohen et al.. 
The same sample approximation strategy can straight forwardly be applied to estimate $\RC$ and $\RNCM$. For estimating the former it can be directly applied for $f^*_{ u_f, \theta}$  and its extended set of classes $\mathcal{\tilde{Y}}$. For estimating $\RNCM$ the label space is modified by uniting the predicted and the uncertainty class.

Now, that we explained our approach and sampling procedure, we can also describe the related work by 
\citet{sheikholeslami2022denoised} that also leveraged a rejection-based classifier in the randomized smoothing setting 
in more detail. Here, the rejection class, which is similar in spirit to our uncertain class but based on an additional classifier and not on evaluating the uncertainty of the base classifier) was only defined for the base classifier and not used in the smoothed counterpart.  The purpose here was to derive better estimates of the probability
$p_{c_{B, \theta}}$ by approximating it through $1 - (\underline{p_{c_{A, \theta}}+ p_{v_{\theta}}})$. 

\subsection{Statistical guarantees}
The following proposition shows that by using this procedure we inherit the statistical guarantees of the original work. 
\begin{proposition}
Let \textsc{Predict} and \textsc{Certify} be as defined in Algorithm~\ref{alg:predict} and \ref{alg:certify}.  Then it follows:
\begin{enumerate}[label=(\alph*)]
    \item The probability that a class other than $g(x)$ or `NaN' is predicted by \textsc{Predict} is at most $\alpha$.
    \item The error of misidentifying the winner and/or the runner up class in \textsc{Predict} is at most $\alpha$.
    \item From~\citet{cohen_adv_robustness}: With probability at least $1-\alpha$ over the randomness in \textsc{Certify}, if \textsc{Certify} returns a radius $R$, the robustness guarantee $g(x+ \delta) = \hat{c}_A$ is given whenever $\| \delta \|_2 < R$. 
\end{enumerate}
\end{proposition}
\begin{proof}
    \textbf{Case (a)} Follows directly from~\citet{cohen_adv_robustness}, as we did not change the procedure in their \textsc{Predict} algorithm to get an estimate of $\hat{c}_A$.\\
    \textbf{Case (b)} follows directly from applying 
     Theorem 1 of \citet{Hung2016RankVF} to only two ranks.\\
    For proving \textbf{case (c)} we can leverage the results from \citet{cohen_adv_robustness}: First, if $\hat{c}_B =$ NaN, then our \textsc{Certify} reduces to the approach from \citet{cohen_adv_robustness}, where the condition $0.5 < \underline{p_A}$ is replaced by the condition, that $R>0$, which is equivalent, since in this setting,
    \begin{align}
        R = & \frac{\sigma}{2} (\Phi^{-1}(\underline{p_A}) - \Phi^{-1}(1 - \underline{p_A}))  = \sigma \cdot \Phi^{-1}(\underline{p_A}) > 0 \nonumber \\
        & \Longleftrightarrow \underline{p_A} > 0.5\enspace . 
    \end{align}
    Second, if $\hat{c}_B \neq$ NaN we distinguish between two cases: if the derived upper and lower bounds sum up to more than 1, we retain the original approach of \citet{cohen_adv_robustness} and hence their guarantee. In the other case we estimate lower and upper bounds based on Bonferroni corrected confidence intervals with $\alpha/2$, such that with probability $1- \alpha$ over the sampling of the noise, $\underline{p_A} \leq \P(f(x+ \epsilon) = \hat{c}_A)$ and $\overline{p_B} \geq \P(f(x+ \epsilon) = \hat{c}_B)$ and $R>0$. 
\end{proof}

\section{Experiments}\label{sec:experiments}
There are several general information about the experiments accompanying the following subsections which we now explain jointly. We refer the reader to the Appendix for more details. 

\subsection{Datasets and models}
All experiments are conducted on the well known CIFAR10~\citep{cifar10} and ImageNet~\citep{imagenet} datasets. We use the first 1,000 samples of the CIFAR10 test set, and every 55th example of the ImageNet test set (as it is ordered according to the different classes) as a validation dataset on which we set the uncertainty threshold $\theta$, such that at most 1\% of benign validation accuracy is lost as done by~\citet{Stutz_CCAT}\footnote{The resulting threshold values and accuracies for all uncertainty functions, models, and datasets are reported in \Cref{sec:appendix_setup}.}. For the actual evaluation, we use every 10th or every 50th example from the respective test sets. 

The second most important aspect of our evaluation is that we did not retrain any of the used models but utilized the provided model checkpoints linked in the GitHub repositories kindly provided by the authors of the respective papers. 
That means, when referring to \textit{RandSmooth} we used the 
ResNet110 model with noise $\sigma = 0.25$ for CIFAR10 and the ResNet50 
with $\sigma = \{0.25, 0.5 \}$ for ImageNet provided by~\citet{cohen_adv_robustness}. With  \textit{SmoothAdv}  we label the experiments based on models trained by~\citet{Salman}, where we chose 
from their extensive model catalog the PGD 2 steps trained ResNet110 model with weight equal to $1.0$ and $\epsilon = 64$ for CIFAR10 and the ResNet50 model trained with the decoupled direction and norm (DDN)~\citep{DDN} approach
with 2 steps, $\epsilon = 255 \text{ or } 1025$, and $\sigma=0.25 \text{ or }0.5$ respectively for ImageNet. 
\textit{SmoothCCAT 0.25/0.3} refers to a combination of the models from~\citet{carlini2023certified} and~\citet{Stutz_CCAT} on CIFAR10 - more precisely, we replaced the classifier originally used in~\citet{carlini2023certified} by the pretrained ResNet-20 CCAT model of~\citet{Stutz_CCAT} evaluated on $\sigma  \in \{0.25,0.3\}$ respectively. This classifier was trained by letting
the prediction in the direction of the nearest adversarial example in the $\ell_{\infty}$-ball of size $0.03$  decay to a uniform distribution. Last but not least, 
we use 
\textit{SmoothViT} to refer to 
a combination of~\citet{carlini2023certified} and the original vision transformer (ViT)~\citep{vit_model}, implemented by~\citet{rw2019timm} and trained on ImageNet. This was inspired by a recent paper~\citep{galil2023framework}, which found that ViT models, especially ViT-L/32-384, excel in OOD detection. 

\begin{table*}[htb]
    \centering
         \resizebox{\linewidth}{!}{
    \begin{tabular}{l c c c c c c c c c c c c c c c}
    \toprule
      &  \multicolumn{15}{c}{CIFAR10}\\
      & 
               $\RC$ & $\Rsup$ & $\RNCM$ & $\RC$ &$\Rsup$ & $\RNCM$ & $\RC$ &$\Rsup$ & $\RNCM$ & $\RC$ &$\Rsup$ & $\RNCM$ & $\RC$ & $\Rsup$ & $\RNCM$  \\
                           \cmidrule(r){2-4} \cmidrule(r){5-7} \cmidrule(r){8-10} \cmidrule(r){11-13} \cmidrule(r){14-16} 
      $\ell_2$-radius & \multicolumn{3}{c}{0.0} & \multicolumn{3}{c}{0.2}  & \multicolumn{3}{c}{0.4} & \multicolumn{3}{c}{0.6} & \multicolumn{3}{c}{0.8}    \\

     \midrule    
    RandSmooth 0.25 & 73.67 & 74.78 & 73.67 &
                62.11 & 62.89 & 63.67 &
                47.89 & 48.11 & 48.89 &
                34.67 & 35.00 & 35.22 &
                21.56 & 21.67 & 21.89\\
    SmoothAdv 0.25 & 76.00 & 76.89 & 76.00 &
                68.33 & 69.67 & 70.00 &
                58.89 & 60.22 & 60.89 &
                48.33 & 49.11 & 49.67 &
                37.89 & 38.22 & 38.22\\
    SmoothCCAT 0.25  & 75.22 & 75.89 & 75.22 &
                62.78 & 64.11 & 64.89 &
                44.11 & 47.78 & 48.78 &
                27.89 & 34.00 & 34.89 &
                14.00 & 23.67 & 24.44 \\
    SmoothCCAT 0.3  & 70.44 & 71.67 & 70.44 &
                58.00 & 60.00 & 61.67 &
                42.44 & 45.44 & 46.67 &
                28.89 & 33.11 & 34.33 &
                19.11 & 24.44 & 25.56 \\
    \midrule
    & \multicolumn{15}{c}{ImageNet}\\
 $\ell_2$-radius  & \multicolumn{3}{c}{0.0} & \multicolumn{3}{c}{0.2}  & \multicolumn{3}{c}{0.4} & \multicolumn{3}{c}{0.6} & \multicolumn{3}{c}{0.8}    \\
            \cmidrule(r){2-4} \cmidrule(r){5-7} \cmidrule(r){8-10} \cmidrule(r){11-13} \cmidrule(r){14-16}  
    RandSmooth 0.25 &68.60 & 69.30 & 68.60 &
            61.50 & 62.00 & 62.00 &
            54.80 & 55.70 & 56.00 &
            48.30 & 48.60 & 48.80 &
            40.20 & 40.50 & 40.60\\
    SmoothAdv 0.25  &  66.20 & 67.20 & 66.20 &
            62.30 & 63.30 & 62.70 &
            58.60 & 59.30 & 59.10 &
            55.10 & 55.60 & 55.70 &
            49.20 & 49.50 & 49.50\\
    \cmidrule(r){2-16}
     $\ell_2$-radius  & \multicolumn{3}{c}{0.0} & \multicolumn{3}{c}{0.2}  & \multicolumn{3}{c}{0.4} & \multicolumn{3}{c}{0.6} & \multicolumn{3}{c}{0.8}    \\
            \cmidrule(r){2-4} \cmidrule(r){5-7} \cmidrule(r){8-10} \cmidrule(r){11-13} \cmidrule(r){14-16}  
    RandSmooth 0.5 & 60.60 & 61.10 & 60.60 &
            56.60 & 56.70 & 56.70 &
            52.00 & 52.40 & 52.20 &
            47.30 & 47.50 & 47.60 &
            43.20 & 43.30 & 43.40 \\
    SmoothAdv 0.5 &  57.80 & 58.60 & 57.80 &
            55.40 & 56.20 & 55.90 &
            52.30 & 52.70 & 52.70 &
            49.50 & 50.10 & 50.20 &
            45.70 & 46.20 & 46.50\\
    SmoothViT 0.5 & 59.50 & 60.30 & 59.50 &
            54.90 & 55.60 & 55.50 &
            49.40 & 50.20 & 50.60 &
            43.70 & 44.20 & 44.60 &
            38.20 & 38.50 & 38.80\\
    \bottomrule 
    \end{tabular} }
    \caption{\textbf{Certified accuracies with respect to $\ell_2$-norm perturbations when using the entropy of the base classifier as uncertainty function.} The higher, the better.}
    \label{tab:cert_acc}
\end{table*}

\subsection{Evaluation of the sampling strategy}
\label{subsec:eval_strategy}
Before analyzing the new robustness guarantees, we set the baseline by investigating how much our approach of deriving an estimate for the runner-up class already improves over the baseline. In \Cref{fig:emp_approx_cifar} we display certified accuracies resulting from the originally used one-vs-all approach (solid lines) vs.  the ones resulting from our approach (dashed lines) described in \Cref{sec:estimates}. We see, that as expected our approach  leads to minor improvements over the baseline. Those benefits result from inputs for which the smoothed classifier places probability mass over 
multiple classes and hence our sample based estimation of $\overline{p_B}$ is tighter than $1-\underline{p_A}$, or from
samples for which the classifier in the original approach abstains 
from predicting due to the constrain $\underline{p_A}>0.5$. 
Given these findings, we recommend practitioners to employ the proposed two-step sampling strategy for deriving robustness guarantees, and we will use our $R$, i.e., the dashed lines, as reference, in the remaining of our paper.

\begin{figure}[htb]
\centering    \includegraphics[width=0.9\linewidth]{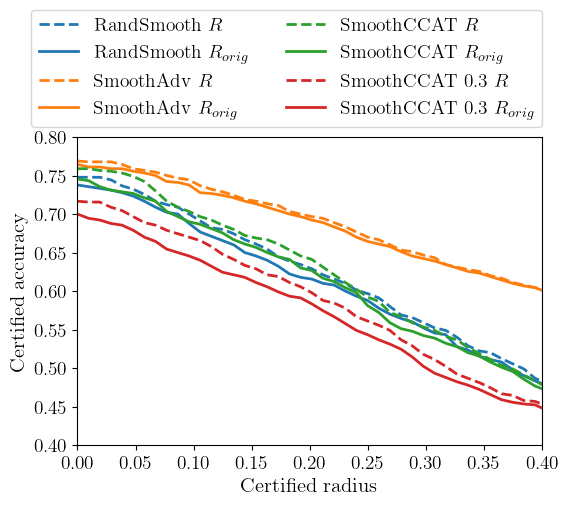}
    \caption{ 
    \textbf{Certified accuracy resulting from the changed certification procedure used for estimating the required bounds.} Models were trained on CIFAR10.
    $R_{orig}$ corresponds to the one-vs-all method of~\citet{cohen_adv_robustness} and $R$ to our proposed scheme with calculating $\overline{p_B}$. A slight increase can be observed for all methods.}
    \label{fig:emp_approx_cifar}
\end{figure}

\subsection{ Robustness evaluation of uncertainty aware classifiers}
\label{subsec:novel_radii}
In this section we investigate the robustness guarantees resulting from the proposed uncertainty aware classifier in practice when employing the entropy of the base classifier as
uncertainty function.
We start by calculating the (approximated) certified test set accuracy for different $\ell_2$-norm perturbation, that is, we estimated the percentage of test samples that were classified correctly, not abstained from predicting, and whose calculated robustness radii was above the $\ell_2$-norm given in the columns. 
Results for the different models and the two data sets are reported in \Cref{tab:cert_acc}.
First, we verify, that the certified accuracies for the standard robustness radius $\Rsup$ are in line with previously reported results. 
Comparing results for $\Rsup$ between models shows that 
for small $\ell_2$-norm perturbations 
the stacked solutions SmoothCCAT and SmoothViT result in comparable high benign accuracy and reasonably high certified accuracies for higher norm perturbations, as it is also visible from the results shown in \Cref{fig:emp_approx_cifar}. This demonstrates the robustness benefits of the novel approaches that combine well-trained classifiers with diffusion models.

We now investigate the certified accuracies resulting from the novel robustness guarantees. First, it can be observed that the benign accuracy (corresponding to the values reported for a  $l_2$-radius of $0.0$) 
for uncertainty aware classifiers is slightly smaller than when not taking uncertainty into account. This is due to examples for which the classifier abstains from making a prediction due to high uncertainty. This effect can also lead to slightly smaller certified accuracies with the uncertainty aware guarantees for smaller values of the  $l_2$-radius. More often however and generally for lager $l_2$-radii, we observe the ordering that we would expect: the certified accuracy for $\RC$ (i.e., the percentage of samples for which it is guaranteed that the prediction in the specified radius stays the same and confident) is smaller than the certified accuracy for $\Rsup$, and the certified accuracy for $\RNCM$ (i.e., the percentage of examples for which it is guaranteed that no wrong class is predicted in the specified radius with low uncertainty) is higher.
The differences are most distinct for SmoothCCAT models. 
This can be explained by the fact that
in contrast to the other methods which were trained to  
make consistent predictions
under Gaussian noise, the CCAT classifier of \citet{Stutz_CCAT} was specifically trained to have a decreasing prediction confidence away from training samples.
This makes larger modifications to input necessary to lead the model in making a wrong (but confident) prediction, and consequently $\RNCM$ is high. At the same time also the area in which predictions are correct and certain is decreased as reflected by a smaller $\RC$.

In summary, the analysis demonstrates that the proposed robustness guarantees allow for a systematic analysis of the robustness properties of uncertainty equipped classifiers.


\begin{table}[htb]
    \centering
         \resizebox{0.7\linewidth}{!}{
    \begin{tabular}{l c c  }
    \toprule
    CIFAR10 &   $\frac{\RC- \Rsup}{\Rsup}$ &$\frac{\RNCM- \Rsup}{\Rsup}$   \\
     \midrule    
    RandSmooth 0.25 & 
                 \phantom{0}-2.67\% & \phantom{0}3.86\% \\
    SmoothAdv 0.25 & 
                 \phantom{0}-2.40\% & \phantom{0}2.19\%\\
    SmoothCCAT 0.25  & 
                     -10.80\% & \phantom{0}6.36\%\\
    SmoothCCAT 0.3  & 
                 \phantom{0}-8.60\%& 13.67\% \\
    \midrule
    ImageNet &&    \\
    \midrule
    RandSmooth 0.25 &
-1.22\% & 1.03 \%  \\
    SmoothAdv 0.25  &  
 -0.81\% & 0.21 \%  \\
    RandSmooth 0.5 & 
 \phantom{-}0.49\% & 1.56 \% \\
    SmoothAdv 0.5 & 
-1.72\% & 0.77 \% \\
    SmoothViT 0.5 &  
-1.31\% & 2.97 \% \\
    \bottomrule 
    \end{tabular} }
    \caption{\textbf{Average per sample change of $\RC$ and $\RNCM$ when compared to $\Rsup$} using the entropy as uncertainty measure. Higher is better. 
}
    \label{tab:radii_comparison}
\end{table}
\subsection{
Sample-wise evaluation of the uncertainty-aware robustness guarantees
} 

When augmenting a classifier with the ability to also output 
uncertainty, one would ideally observe low uncertainty when making the right prediction, 
and observe high uncertainty when predicting another class close to the input. Using such a classifier as the base for the smoothed classifier would translate to having an $\RC$ as large as possible, while at the same time having  a high $\RNCM$.
To get a better picture of the novel guarantees we evaluate their relative difference to the robustness guarantee without leveraging uncertainty on the individual input sample basis. 
Results are shown in Table~\ref{tab:radii_comparison}, and different substantially between models. Its generally observable that models with a relatively large $\RNCM$ have a smaller $\RC$, which indicates a generally larger amount of probability mass on areas belonging to the uncertainty class  under these models.
For some models the relative increase of $\RNCM$ compared to $\Rsup$ is larger than the relative decrease for $\RC$ compared to $\Rsup$. 
For other models it is the other way around. This can have multiple reasons, for example as discussed in \Cref{sec:including_uncertainty}, that the probability mass for the correct class was already close to 1, or that the probability mass of the different classes was shifted unevenly to the uncertainty class. 

Overall, the differences between the radii are more eminent for models trained on CIFAR10 than ImageNet, which can be tailored to the impact of the Gaussian noise on images of different resolutions. Classifying samples of CIFAR10 under the same amount of noise as for a high resolution ImageNet picture is more challenging, and therefore we are expecting more uncertain predictions for noisy versions of CIFAR10 than for ImageNet which in return impacts $\RC$ and $\RNCM$. 

Interestingly, for RandSmooth 0.5 on ImageNet, we observe that $\RC$ actually increases over $\Rsup$, which is tailored to our discussion in section~\ref{sec:including_uncertainty} which explained under which conditions including uncertainty can improve over $\Rsup$. This shall be investigated in the next subsection in greater detail. 

\subsection{Including uncertainty can improve robustness guarantee}
As observed in \Cref{tab:radii_comparison} and stated in Corollary~\ref{cor:uncertainty_robust}, it is  possible that the certified robustness radius of an uncertainty-equipped classifier can even be larger or equal to the certified radius of the original classifier.  In this section, we investigate this phenomenon in more detail. 

\begin{table}[tb]
    \centering
    \resizebox{\linewidth}{!}{
    \begin{tabular}{l c c c c }
    \toprule
    & positive $\uparrow$ &  negative $\downarrow$
    & zero  & total \\
    \midrule
    & \multicolumn{4}{c}{CIFAR10 0.25} \\
    \cmidrule(r){2-5} 
    RandSmooth    & 5.40\% & 44.83\% & 49.78\% &667\\
     SmoothAdv    & 2.19\% & 31.14\% & 66.67\% &684 \\ 
     SmoothCCAT & 6.48\% & 86.75\% & \phantom{0}6.77\% &679 \\
     SmoothCCAT 0.3 &8.20\% & 87.46\% & \phantom{0}4.33\% &646  \\
     \midrule 
     &\multicolumn{4}{c}{ImageNet 0.25} \\ 
     \cmidrule(r){2-5} 
     RandSmooth & 1.45\% & 12.92\% & 85.63\% &689 \\
     SmoothAdv  &  0.60\% & 4.08\% & 95.32\% &662 \\
          \midrule 
     &\multicolumn{4}{c}{ImageNet 0.5} \\ 
     \cmidrule(r){2-5} 
     RandSmooth & 3.12\% & 12.97\% & 83.91\% &609 \\
     SmoothAdv  &  0.69\% & 13.94\% & 85.37\% &581 \\
     SmoothVIT & 6.12\% & 24.96\% & 68.93\% &605 \\
    \bottomrule 
    \end{tabular} 
    }
    \caption{\textbf{
    Percentage of correctly classified samples for which the difference between $\RC$ and 
    $\Rsup$
    is positive, negative, or zero, respectively}.
    We use the entropy of the base classifier as the uncertainty function and chose (approximately) the minimum $\theta$ such that we only lose 1\% benign accuracy on a validation set.} 
    \label{tab:comparison-RCC}
\end{table}
While in the previous subsection, we only looked at the average per sample difference between $\RC$ and $\Rsup$, we now report in Table~\ref{tab:comparison-RCC} how often  
i) $\RC - \Rsup >0$,
ii) $\RC - \Rsup < 0$, and 
iii) $\RC - \Rsup = 0$,
that is, how often the certified radius of the uncertainty-equipped classifier is larger, smaller, or equal to the certified radius of the original classifier. We perform the evaluation
on correctly (and with certainty in the case of the uncertainty-equipped classifier) classified images.
For a huge part of the 
samples 
the certified robustness radius does not change when uncertainty is included in the classifier. 
Inspecting these samples more closely, 
we find that almost no noisy input variations are classified as uncertain (i.e. result in the prediction of the uncertainty class $v_{ \theta}$),
such that including uncertainty does not affect the robustness radius of these data points. 
Naturally, for those samples with increased or decreased radius it holds 
$p_{v_{ \theta}} \neq 0$. It is worth mentioning, that for all datasets and methods there exists a (small) subset of samples for which the robustness radius increases when using uncertainty. 
This demonstrates that the phenomenon discussed in section~\ref{sec:including_uncertainty} indeed occurs in practice and that in future work asymmetric uncertainty can be pushed to excel $\RC$ further.
For the samples with a decreased radius 
we observe that more than half of the samples do not have a runner-up class but if they do, the runner-up class is actually the uncertain class for approximately every fourth sample. In these cases, the eminent threat is in shifting to the uncertain class which might be - depending on the use case - not a failure case. 
In summary, these results show that including an uncertainty function with a reasonable chosen threshold does not need to impact the benign certified radius and that it can even be beneficial. 

\subsection{Evaluating the guarantees for different uncertainty measures and thresholds}
As follows from the definition, an uncertainty-equipped classifier
depends on the uncertainty function used. 
In the experiments discussed in  the previous section, we used the entropy of the base classifier as the uncertainty function. 
To get a feeling for how the choice of the uncertainty function influences the results, we also estimated the certified accuracy  when  using the prediction confidence and the prediction margin as uncertainty functions.
To allow for comparability, $\theta$ was (again) chosen such
that the validation accuracy on a validation set looses at most $1\%$ in all cases. The results are shown in \Cref{fig:radii_uncs}. The prediction confidence and the entropy often lead to similar certified accuracies, up to an degree that resulted in overlapping graphs. For all models the margin achieved the largest certified accuracy with respect to $\RNCM$. At the same time the certified accuracy with respect to  $\RC$ was smaller only for the models RandSmooths and SmoothAdvs. This results demonstrate that the choice of the uncertainty measure clearly impacts the robustness of the smoothed classifier and that some uncertainty measure can lead to a larger accuracy with respect to $\RNCM$ while not impacting the amount of samples being correctly classified with high certainty.

For a closer inspection we also evaluated the novel robustness radii resulting from the additional uncertainty functions on a per-sample bases, see
\Cref{tab:radii_comparison_conf_margin}. 
In contrast to the results of the entropy, here the average per sample increases of $\RNCM$ are always (in the case of the margin) or mostly (for the confidence) higher than the accuracy losses when looking at $\RC$. Remarkably high are the results for $\RNCM$ when using the margin, making this an explicitly well suited uncertainty measure when the eminent thread of adversarial examples is posed by a confident misclassification.

Finally, we conducted an investigation of the dependency of the guarantees from the threshold chosen for identifying predictions as certain. We investigated $\theta$ values resulting in a decrease of the validation
set accuracy of maximally by 0.5\%, 2\% and 5\%, respectively. We show the certified accuracy  for SmoothAdv and different uncertainty functions in \Cref{fig:radii_theta}. As expected, the certified accuracy  w.r.t. $\RNCM$ as well as to $\RC$ approaches $\R$ with increasing threshold. Nevertheless, the robustness differences emerging from different uncertainty functions remain clearly visible.

\begin{table}[htb]
    \centering
         \resizebox{0.9\linewidth}{!}{
    \begin{tabular}{l c c c c }
    \toprule
    & \multicolumn{2}{c}{Confidence} & \multicolumn{2}{c}{Margin}\\
    CIFAR10 &   $\frac{\RC- \Rsup}{\Rsup}$ &$\frac{\RNCM- \Rsup}{\Rsup}$  &   $\frac{\RC- \Rsup}{\Rsup}$ &$\frac{\RNCM- \Rsup}{\Rsup}$  \\
     \midrule    
    RandSmooth 0.25 &  -3.82\% & 6.66\% & 
 -9.0\% & 14.27\% \\
    SmoothAdv 0.25 & 
                 -3.85\% & 4.55\% & 
 -7.44\% & 10.66\% \\
    SmoothCCAT 0.25  &-9.66\% & 8.14\% & 
 -14.38\% & 15.52\%\\
    SmoothCCAT 0.3  & 
        -9.83\% & 17.33\% & 
 -11.75\% & 20.93\% \\
    \midrule
    ImageNet &&    \\
    \midrule
    RandSmooth 0.25 &
 -1.09\% & 1.62\% & 
 -4.98\% & 9.17\%  \\
    SmoothAdv 0.25  &  
  -1.37\% & 1.01\% & 
 -4.29\% & 6.78\% \\
    RandSmooth 0.5 & 
 -0.56\% & 2.23\% & 
 -3.74\% & 9.57\%\\
    SmoothAdv 0.5 & 
-1.28\% & 0.87\% & 
 -4.71\% & 7.12\% \\
    SmoothViT 0.5 &  
 -2.36\% & 6.09\% & 
 -1.40\% & 13.23\% \\
    \bottomrule 
    \end{tabular} }
    \caption{\textbf{Average per sample change of $\RC$ and $\RNCM$ when compared to $\Rsup$.} Higher is better. Corresponding results for the entropy are given in \Cref{tab:radii_comparison}.}
\label{tab:radii_comparison_conf_margin}
\end{table}

\begin{figure}[htb]
\subfloat[RandSmooth $\sigma = 0.25$]{\includegraphics[width= 0.245\textwidth]{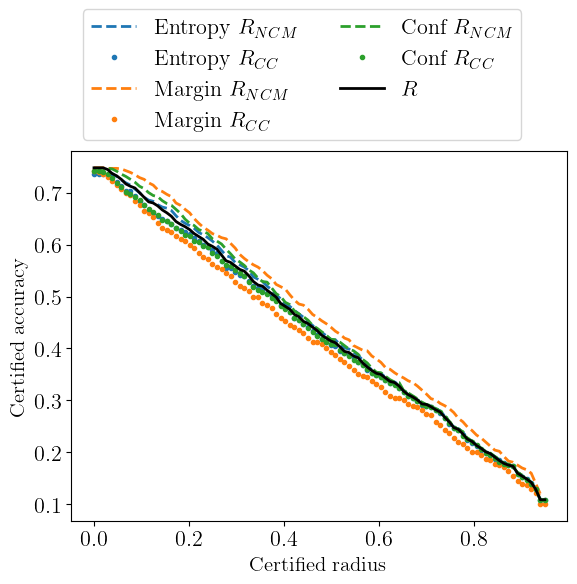}}
\subfloat[SmoothAdv $\sigma = 0.25$]{\includegraphics[width= 0.245\textwidth]{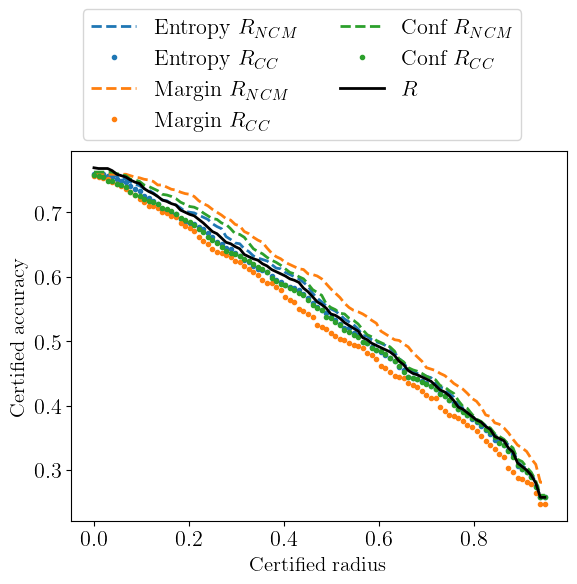}}
\quad
\subfloat[SmoothCCAT $\sigma = 0.25$]{\includegraphics[width= 0.245\textwidth]{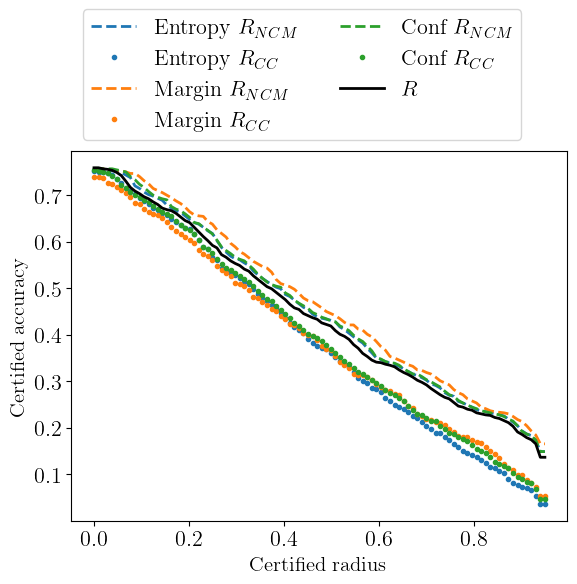}}
\subfloat[SmoothCCAT 0.3 $\sigma = 0.25$]{\includegraphics[width= 0.245\textwidth]{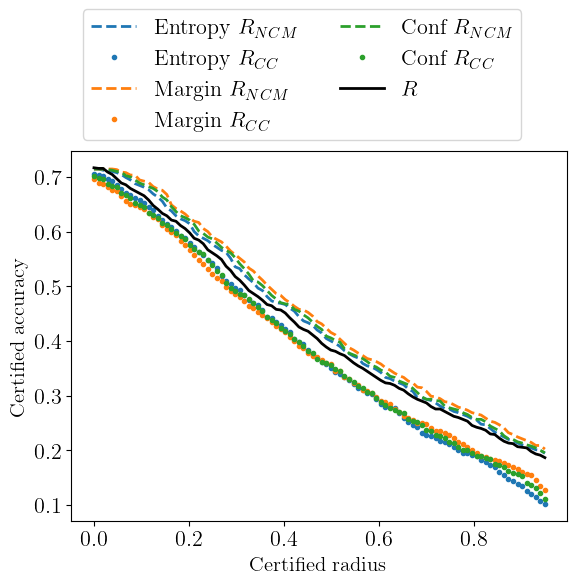}}
\caption{\textbf{The effect of
different uncertainty functions} on the certified accuracy, where each threshold is set such that at most 1\% benign validation set accuracy is lost.
}
\label{fig:radii_uncs}
\end{figure}

\begin{figure}[htb]
\subfloat[SmoothAdv 0.25 Confidence]{\includegraphics[width= 0.245\textwidth]{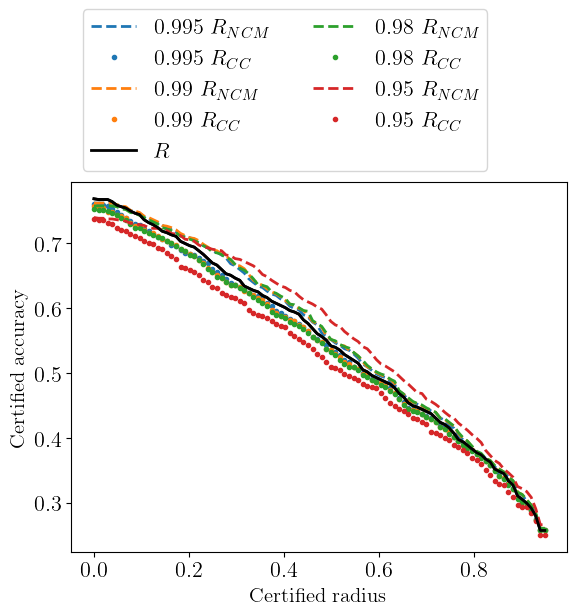}}
\subfloat[SmoothAdv 0.25 Entropy]{\includegraphics[width= 0.245\textwidth]{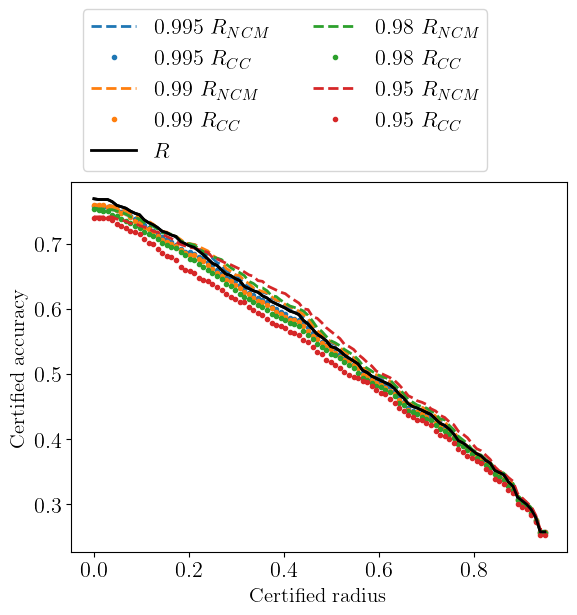}}
\quad
\centering
\subfloat[SmoothAdv 0.25 Margin]{\includegraphics[width= 0.245\textwidth]{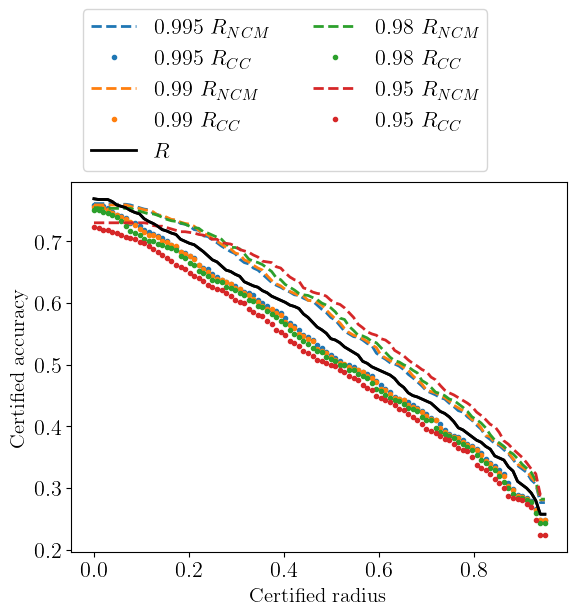}}
\caption{\textbf{The effect of different $\theta$ values} on the certified accuracy for SmoothAdv with different uncertainty functions. $\theta$ values are chosen to allow for varying percentage of benign accuracy decreases.}
\label{fig:radii_theta}
\end{figure}

\subsection{Demonstrating advantages of uncertainty equipped smoothed classifiers on OOD data}
\label{subsec:ood}
In the previous sections, we have 
just looked at correctly classified in-distribution (ID) test samples. The typical application for uncertainty investigates the out-of-distribution (OOD) performance.  
There are three aspects particularly interesting in this setting which are depicted in Figure~\ref{fig:ood}: how often do we abstain from making a prediction (\textcolor{tab_red}{\ding{58}}), how often is the uncertain class 
the predicted class (\textcolor{tab_blue}{$\bigstar$}), and how often is the uncertain class the runner-up class (\LARGE{\textcolor{tab_orange}{$\bullet$}}\normalsize ) in comparison to the ID data. For this comparison, we used CIFAR100 as OOD data, and only selected images labeled as \textit{telephone} or \textit{keyboard} for a more visible distinction to CIFAR10 in-distribution data. This results in an OOD set size of 200. \Cref{fig:ood} plots on the x-axis the fraction of OOD samples against the fraction of the ID samples on the y-axis. This result is based on  RandSmooth on CIFAR10 with the margin as an uncertainty function but similar results are also archived with the other methods (compare \Cref{sec:appendix_experiments}). 
We see that on OOD data we get more often an ``uncertain'' prediction and also the runner-up class is more often  ``uncertain''. Additionally, we observe that there are fewer cases where $p_{v_{\theta}}(x) = 0$ (\LARGE{\textcolor{tab_green}{$\blacktriangledown$}}\normalsize) and where we do not get an estimate of $\hat{c}_B$ from \textsc{Certify}. 
This plot shows, that including uncertainty does not only increase the radius in which it is guaranteed that a shift of the input does not lead to the confident prediction of another class, but also transfers desired properties of the base classifier over to the smoothed classifier, hence making the smoothed classifier also more robust with regard to OOD data.
Note that this is not given if we adopt the framework of \citet{sheikholeslami2022denoised} where the smoothed classifier does not have an uncertain class.

\begin{figure}
    \centering
    \includegraphics[width= 0.4\textwidth]{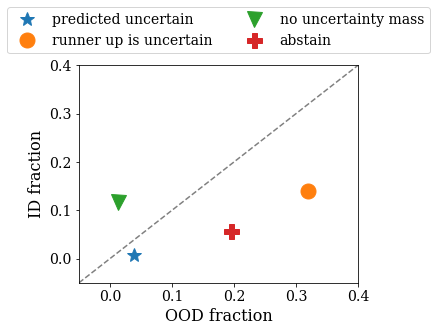}
    \caption{
    \textbf{Our approach is more likely to abstain and/or predict the uncertain class on OOD (vs. ID) examples.}
    CIFAR10 for RandSmooth: fraction of samples in the ID and OOD dataset where the uncertain class is predicted, the runner-up class is uncertain, no runner-up class was identified, the uncertain class has no probability mass, and where we abstain from making a prediction. 
    }
    \label{fig:ood}
\end{figure}

\section{Limitations of our work.}\label{sec:limitations}
Our work is based on~\citet{cohen_adv_robustness} and therefore inherits the limitations of their approach, too. That is, to identify and calculate the upper and lower bounds on the probabilities we need to perform extensive sampling which makes the approach computationally costly. However, there are possible ways to speed this process up for example by the iterative sampling procedure proposed by~\citet{mark2022boosting}. Furthermore, we believe that all results can be improved by including the new insights about uncertainty of the base classifier directly into the training. However, this would again result in significant training costs and we already can prove our point in existing methods which makes our results even stronger because benefits are general and not the result of specific engineering. Besides the computational side, we would like to highlight that certified robustness with randomized smoothing gives a high probability guarantee, but there is still a chance left for making a mistake.

\section{Conclusion}\label{sec:conclusion}
In this work, we have merged two approaches developed for increasing network robustness: certified robustness via randomized smoothing and uncertainty-based rejection of predictions. 

By equipping the base classifier with an uncertainty function before smoothing,
we propose a new certification scheme that leads to the derivation of two novel robustness guarantees. 
The first, $\RC$, specifies an $\ell_2$-ball around an input in which predictions will not change and uncertainty remains low. The second, $\RNCM$, 
specifies a region in which predictions will not be wrong (assuming the input to be classified correctly) and indicate low uncertainty. 
The former provides guarantees against standard adversarial attacks as well as those designed to push examples to uncertain regions. It is surprisingly often as large as the standard robustness radius and sometimes even larger while guaranteeing not only consistent but also confident predictions. 
The latter is of interest for safety-critical applications and is 
up to 20.93\% larger than the standard robustness radius for classifiers not allowing for uncertainty based rejection.

This novel robustness framework allows us to
identify desirable properties of uncertainty functions, namely, asymmetry in the sense of less uncertainty in regions close to the input corresponding to the same class  and higher uncertainty in areas corresponding to different classes. 
Moreover, by incorporating uncertainty into the smoothed classifiers we get increased OOD robustness of the smoothed classifier for free. 

The introduced robustness certificates represent
a mathematically grounded way to formally
compare the robustness resulting from different network architectures and uncertainty measures and thus are useful tools for a principled robustness evaluation in future work. 
Given recent advantages in diffusion models and the possibility to get ``certified robustness for free''~\citep{carlini2023certified} we can combine more powerful uncertainty models with denoising operations. 
We expect the first results presented here to extensively improve with advances in both research fields. 

\section*{Acknowledgment}
This work is funded by the Deutsche Forschungsgemeinschaft (DFG, German Research Foundation) under Germany’s Excellence Strategy - EXC 2092 CASA - 390781972.

\bibliographystyle{abbrvnat}
\bibliography{library}

\newpage
\include{supplement}

\end{document}

%% file: supplement.tex
\appendix

\subsection{Experimental setup}\label{sec:appendix_setup}
All our experiments can be reproduced by using a Nvidia A40 GPU. Note that experiments with SmoothViT needed the full capacity with a batch size of 50, which means we predict during evaluation 50 noisy versions of the same input at once.

\subsubsection{Datasets and models}
As stated in the main text, we have not retrained any of the models but used existing frameworks. How to set up these experiments can be found in the \textit{readme} in the code base.
For the models \textit{SmoothCCAT} and \textit{SmoothViT} we used the setting of~\citet{carlini2023certified}, which implies that we used the class unconditioned diffusion models by~\citet{improved_diff_cifar10} for SmoothCCAT on CIFAR10 and \citet{diff_imagenet} for SmoothViT on ImageNet as the denoiser. The numbers behind the names display the amount of Gaussian noise $\mathcal{N}(0, \sigma^2)$ added during training and evaluation.

\subsubsection{Thresholds and accuracies}
In the main paper, we state that we set the threshold  
$\theta$ such that the accuracy on the respective validation set is at most reduced by 1\%. This was calculated in the following way: 

We first set $\theta = \sup(\mathcal{U})$ and draw 1,000 noisy samples for each image and predict the labels of those (which essentially corresponds to \textsc{SampleUnderNoise}$(f, x, \sigma, 1000)$). Without conducting a hypothesis test, we output as a prediction the class with the most class counts in \texttt{counts}, 
i.e., the majority class, and report the resulting accuracy over the validation set. 
We then gradually increase $\theta$ in equidistant 1,000 steps from 0 to 1 for the margin, from 0.1 to 1 for prediction confidence and from $-\log( \text{number of classes})$ to 0 for the entropy 
and repeat the above-described 
procedure of calculating the resulting accuracy. 
If the accuracy drops below 99\% of the original accuracy with $ \theta =\sup(\mathcal{U})$ we stop and report the last $\theta$ for which this was not the case. The thresholds are given in Table~\ref{tab:theta}. Note, that due to the definition of an uncertainty-equipped classifier the uncertainty functions 
corresponding to prediction and margin are actually their negative values. 
In Table~\ref{tab:accuracie} we report the resulting test set accuracies and observe that it is oftentimes even below the 1\% loss we accepted on the validation set.
Also note, that these accuracies are higher than the ones reported in the main paper, as for these results there is not a statistical test conducted to lead to abstaining from predicting.

\begin{table}[tb]
    \centering
    \begin{tabular}{c l c c c c}
    \toprule
     &  & Prediction  & Entropy & Margin  \\
     \midrule
     \multirow{4}{*}{\begin{sideways}CIFAR10\end{sideways}}  
     &RandSmooth 0.25 & 0.39 & -1.545& 0.109\\
     &SmoothAdv 0.25 & 0.421 & -1.474 & 0.124\\
     &SmoothCCAT 0.25 & 0.449 & -1.547& 0.257\\
     &SmoothCCAT 0.3 & 0.490 & -1.504& 0.242\\
     \midrule
     \multirow{5}{*}{\begin{sideways}ImageNet\end{sideways}}     
     & RandSmooth 0.25& 0.159 & -3.875 & 0.045\\
     & SmoothAdv 0.25 &0.159& -4.283& 0.027\\
      \cdashline{2-5}
     & RandSmooth 0.5& 0.113 & -4.594& 0.028\\
     & SmoothAdv 0.5 & 0.125& -4.304& 0.024\\
     & SmoothViT 0.5 & 0.160& -4.704 & 0.047 \\   
    \bottomrule 
    \end{tabular}
    \caption{\textbf{Values of the  threshold  $\theta$ } such that the benign accuracy of the smoothed classifier on the validation set does not loose more than 1\% accuracy when using uncertainty. 
    }
    \label{tab:theta}
\end{table}

\begin{table}[tb]
    \centering
    \begin{tabular}{c l c c c c }
    \toprule
     &  &Original  &Prediction& Entropy& Margin \\
     \midrule
     \multirow{4}{*}{\begin{sideways}CIFAR10\end{sideways}}  
     & RandSmooth 0.25 & 77.89 & 76.89 & 77.78 & 76.78  \\
     &SmoothAdv 0.25 & 79.90& 77.78 &77.78 &77.78\\
     &SmoothCCAT 0.25 & 79.00 &78.11 &77.89&77.67\\
     &SmoothCCAT 0.3 &74.33 &73.67&73.89&73.56\\
     \midrule
     \multirow{5}{*}{\begin{sideways}ImageNet\end{sideways}}     
     & RandSmooth 0.25& 70.5 & 70.1 & 69.9 & 69.3 \\
     & SmoothAdv 0.25 &67.5 &65.2 & 66.5 & 66.6\\
    \cdashline{2-6}
     &RandSmooth 0.5 &62.4 & 62.1 & 62.3 & 61.9\\
     &SmoothAdv 0.5& 60.4 & 59.1 & 59.3 & 59.3\\
     &SmoothViT 0.5 &63.7 & 61.9 & 62.2&62.0\\
    \bottomrule 
    \end{tabular}
    \caption{\textbf{Test set accuracies of the smoothed classifier} based on 1,000 samples for different uncertainty functions and thresholds as reported in Table~\ref{tab:theta}. `Original' corresponds to  the smoothed classifier before equipping it with an uncertainty function (or equivalently when  $\theta=\sup(\mathcal{U})).$}
    \label{tab:accuracie}
\end{table}

\subsection{Additional experimental results}\label{sec:appendix_experiments}

In this section, we provide additional experimental results. 

\paragraph{Number of neighboring classes for ImageNet}
We first provide the number of different classes predicted for noisy versions of each image (i.e. the non-zero entries of the \textsc{counts} vector provided by \textsc{SampleUnderNoise}$(f, x, \sigma, 1000)$) from the test set for the models trained on ImageNet, when using the entropy as the uncertainty function in Figure~\ref{fig:classes_imagenet}. Again, as displayed in the main paper, we observe that on average over the used models around 50\% of the noisy samples fall not only into one but multiple classes.

\begin{figure}[tb]
    \centering
    \includegraphics[width=0.48\textwidth]{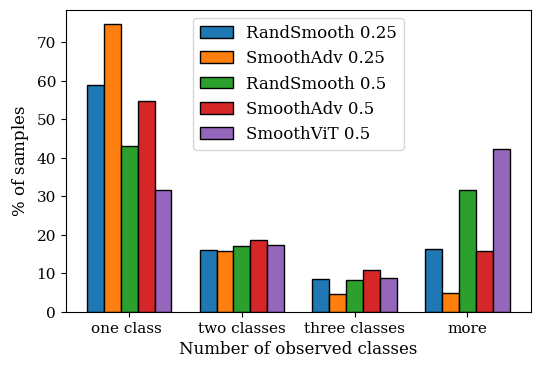}
    \caption{\textbf{Number of different classes} observed when using 1,000 noisy samples per image on ImageNet.}
    \label{fig:classes_imagenet}
\end{figure}

\paragraph{Complementary experiments for \Cref{subsec:ood}: Investigating the OOD case}
\label{sec:app_ood}
Last but not least, we present the remaining results for the other models on the OOD datasets. For ImageNet we used the ImageNet-O dataset~\citep{hendrycks2021nae} as OOD data, where the ViT-L/32-384 model\footnote{We are using this ViT classifier in our experiments for SmoothViT.} was found to excel on OOD detection in a recent paper~\citep{galil2023framework}. 
However, this paper also finds that ImageNet-O can be considered a tough OOD dataset. Since our analysis is based on noisy images small details which make these samples OOD might be destroyed. Indeed we obverse, that there are hardly any differences between OOD and ID for the other models on ImageNet. Only for SmoothViT we find more often that a fraction of noisy samples are predicted as uncertain, which results in more samples with no empty uncertainty mass, e.i. $p_{v_{\theta}}(x)\neq 0$ for the images on OOD than on ID data.

\begin{figure*}[t]
    \centering
    \subfloat{\includegraphics[width= 0.35\textwidth]{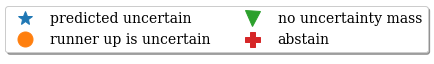}}
    \quad
\addtocounter{subfigure}{-1} 
\subfloat[SmoothAdv 0.25]{\includegraphics[width= 0.32\textwidth]{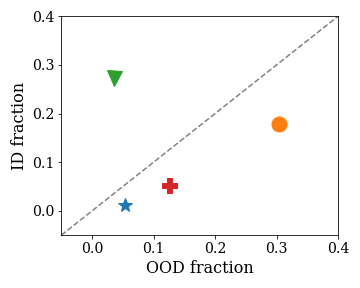}}
    \subfloat[SmoothCCAT 0.25]{\includegraphics[width= 0.32\textwidth]{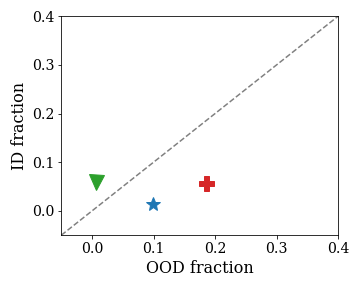}}
    \subfloat[SmoothCCAT 0.3 ]{\includegraphics[width= 0.32\textwidth]{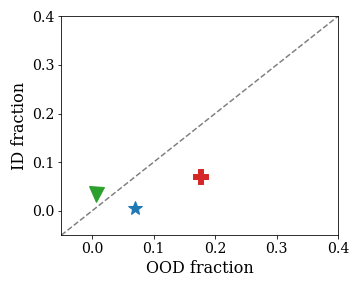}}
    \caption{
    \textbf{Comparing the smoothed classifier's performance on in distribution (ID) to OOD data:} Fraction of samples in the ID and OOD dataset where the uncertain class is predicted, the runner-up class is uncertain, no runner-up class was identified, the uncertain class has no probability mass, i.e. no noisy samples were classified as uncertain and last, where we abstain from making a prediction. Models trained on
    CIFAR10. }
\end{figure*}
    
\begin{figure*}[t]
\centering
\subfloat{\includegraphics[width= 0.35\textwidth]{figs/legende.png}}
    \quad
    \addtocounter{subfigure}{-1} 
    \subfloat[RandSmooth 0.25]{\includegraphics[width= 0.32\textwidth]{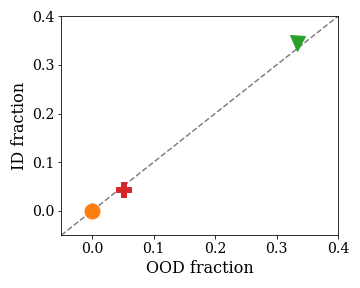}}
    \subfloat[RandSmooth 0.5]{\includegraphics[width= 0.32\textwidth]{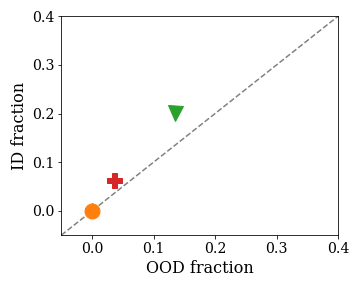}}
    \quad
    \subfloat[SmoothAdv 0.25 ]{\includegraphics[width= 0.32\textwidth]{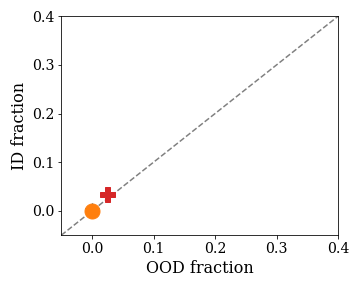}}
    \subfloat[SmoothAdv 0.5]{\includegraphics[width= 0.32\textwidth]{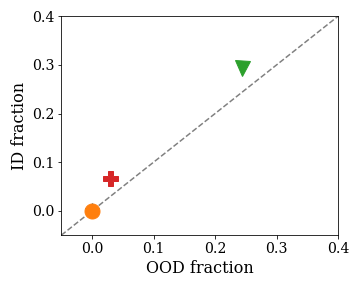}}
    \subfloat[SmoothViT 0.5 ]{\includegraphics[width= 0.32\textwidth]{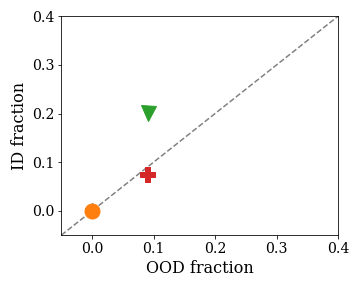}}
    \caption{
    \textbf{Comparing the smoothed classifier's performance on in distribution (ID) to OOD data:} Fraction of samples in the ID and OOD dataset where the uncertain class is predicted, the runner-up class is uncertain, no runner-up class was identified, the uncertain class has no probability mass, i.e. no noisy samples were classified as uncertain and last, where we abstain from making a prediction.  Models trained on ImageNet.}
    \label{fig:id_ood_fraction}
\end{figure*}